%% file: root.tex
\documentclass[letterpaper, 10 pt, conference]{ieeeconf}  

\IEEEoverridecommandlockouts                              

\usepackage{cite}
\usepackage{amsmath,amssymb,amsfonts}

\usepackage{algorithmic}
\usepackage{graphicx}
\usepackage{algorithm,algorithmic}
\usepackage{hyperref}
\usepackage{textcomp}

\allowdisplaybreaks[2]
\input{defs.tex}
\usepackage{mathtools}
\usepackage{pifont}
\usepackage{bbm}
\usepackage{tikz}
\usepackage[utf8]{inputenc}
\usepackage{pgfplots}
\usepackage{setspace}
\usepackage{pict2e,color}
\usepackage{multirow}
\usepackage{booktabs}
\usepackage{verbatim} 
\usepackage{url}
\usepackage{empheq}
\usepackage[normalem]{ulem}
\usepackage[flushleft]{threeparttable}
\usepackage{soul} 
\usepackage{algorithm,algorithmic}
\usepackage{microtype}
\usepackage{subfigure}
\usepackage{etoolbox}
\usepackage{pbox}
\newtheorem{definition}{Definition}
\newtheorem{theorem}{Theorem}
\newtheorem{lemma}{Lemma}
\newtheorem{proposition}{Proposition}
\newtheorem{corollary}{Corollary}

\newtheorem{remark}{Remark}

\newtheorem{assumption}{Assumption}

\makeatletter
\newenvironment{replemma}[1]
  {\@begintheorem{Lemma}{#1 (restated)}}
  {\@endtheorem}

\newenvironment{reptheorem}[1]
  {\@begintheorem{Theorem}{#1 (restated)}}
  {\@endtheorem}

\newenvironment{repcorollary}[1]
  {\@begintheorem{Corollary}{#1 (restated)}}
  {\@endtheorem}
\makeatother

\title{\LARGE \bf
On the Approximation and Convergence of Distributional Policy Gradient Algorithms for Risk-Sensitive Reinforcement Learning
}

\author{Xian Yu, Minheng Xiao, and Lei Ying
\thanks{The work of Xian Yu is supported in part by NSF under grant 2331782. The work of Lei Ying is supported in part by NSF under grants 2112471, 2207548, 2228974, 2240981, and 2331780.}
\thanks{Corresponding author: Xian Yu is with the Department of Integrated Systems Engineering at The Ohio State University, Columbus, OH 43210, USA (e-mail: yu.3610@osu.edu). }
\thanks{Minheng Xiao is with the Department of Integrated Systems Engineering at The Ohio State University, Columbus, OH 43210, USA (e-mail: xiao.1120@osu.edu). }
\thanks{Lei Ying is with the Department of Electrical Engineering and Computer Science at University of Michigan, Ann Arbor, MI 48109, USA (e-mail: leiying@umich.edu).}}

\begin{document}

\maketitle
\thispagestyle{empty}
\pagestyle{empty}

\begin{abstract}
Risk-sensitive reinforcement learning (RL) is crucial for maintaining reliable performance in high-stakes applications. While traditional RL methods aim to learn a point estimate of the random cumulative cost, distributional RL seeks to estimate the entire distribution of it, leading to a unified framework for handling different risk measures \cite{bellemare2017distributional}. However, developing policy gradient methods for risk-sensitive distributional RL is inherently more complex as it often involves finding the gradient of a probability measure. This paper introduces a new distributional policy gradient framework for risk-sensitive RL, where we derive an analytical gradient of the probability measure of the cumulative cost. For practical implementation, we further design a categorical distributional policy gradient algorithm (CDPG) that approximates arbitrary distributions using a categorical family supported on fixed points. Using Conditional Value-at-Risk (CVaR) as the objective, we prove that the proposed CDPG converges to stationary points and establish its iteration complexities under inexact policy evaluation. Through experiments in a stochastic Cliffwalk environment, we demonstrate the effectiveness of the proposed algorithm and highlight the benefits of incorporating risk sensitivity into distributional RL.
\end{abstract}

\section{Introduction}
\label{sec:introduction}
Sequential decision-making under uncertainty is central to complex dynamical systems, including power systems, transportation, robotics, and healthcare. Markov decision processes (MDPs) provide a standard framework for sequential decision-making under uncertainty, where an agent interacts with an environment by choosing actions to optimize a long-term objective. Traditional reinforcement learning (RL) often aims to minimize the expected cumulative cost or maximize the expected cumulative reward~\cite{sutton2018reinforcement}. However, for intelligent autonomous systems operated in highly risky and dynamic environments, such as autonomous driving, healthcare and finance, it is equally or more important to control the \textit{risk} under various possible outcomes. To address this, risk-sensitive RL has been developed to ensure more reliable performance using different objectives and constraints~\cite{heger1994consideration, coraluppi2000mixed, chow2014algorithms, chow2018risk, tamar2015policy, tamar2015optimizing}. Among them, a popular class of objectives is the \textit{coherent risk measures} proposed in \cite{artzner1999coherent}, which satisfy several natural and desirable properties. One of the most widely used coherent risk measures is the Conditional Value-at-Risk (CVaR), also known as Expected Shortfall \cite{rockafellar2002conditional}. Unlike the expectation, CVaR focuses on the worst-case outcomes by computing the expected loss beyond a specified Value-at-Risk (VaR) threshold, thereby providing a principled measure of tail risk \cite{rockafellar2000optimization,rockafellar2002conditional, ruszczynski2006optimization, shapiro2009lectures}.  
In sequential decision-making problems, risk can be assessed either on the total cumulative cost, resulting in a static risk measure, or in a nested manner across stages, resulting in a dynamic risk measure. While dynamic risk measures are time-consistent and naturally admit Bellman equations \cite{ruszczynski2010risk,yu2022risk,yu2023global}, static risk measures are often more intuitive to interpret and model \cite{tamar2015policy,tamar2015optimizing}. In this paper, we focus on optimizing the static CVaR, defined as follows:
\begin{align}
\label{eq:ObjectiveFunction}
\min_\theta \rho(Z_{\theta}^s)
\end{align}
where $Z_\theta^s$ is the random variable representing the cumulative discounted cost along the trajectory following policy $\pi_{\theta}$ starting from state $s$, and $\rho=\rho_{\text{CVaR}}$ represents static CVaR.

Instead of modeling a point estimate of the random cumulative cost $Z_{\theta}^s$, \cite{bellemare2017distributional, bellemare2023distributional,dabney2018distributional,dabney2018implicit} proposed a distributional RL (DRL) framework that models the entire distribution of the cumulative cost and developed distributional dynamic programming methods for risk-neutral and risk-sensitive RL. Recently, \cite{pires2025optimizing} revisited these distributional dynamic programming methods for optimizing a class of statistical functionals of the cost distribution on the stock-augmented MDP. Besides value-based methods, various distributional policy gradient methods have also been proposed \cite{barth2018distributed,ma2020dsacf,singh2020improving,singh2022sample}. However, existing DRL literature either focuses on value-based methods or on policy gradient methods that lack analytical policy gradient expressions and theoretical convergence guarantees. Different from these papers, our work aims to develop distributional policy gradient algorithms with explicit analytical gradients and provable convergence guarantees. Specifically, we introduce a general distributional policy gradient framework that consists of two main steps: (i) \textbf{distributional policy evaluation}, where we estimate the distribution of $Z_{\theta}^s$ (or its corresponding probability measure $\eta_{\theta}^s$) under any given policy $\pi_{\theta}$; and (ii) \textbf{distributional policy improvement}, where we provide an analytical form of its probability measure's gradient $\nabla_{\theta}\eta_{\theta}^s$, which is helpful in deriving the policy gradient for a broad class of risk measures. For practical implementation, we also design a categorical distributional policy gradient algorithm (CDPG) that approximates arbitrary distributions using a categorical family supported on fixed points. We prove that our proposed algorithm can find an $\epsilon$-stationary point in finite time under inexact policy evaluation.

Our main contributions can be summarized as follows. 

\begin{itemize}
    \item In Section~\ref{section:DPG}, we present a general distributional policy gradient theorem (Theorem~\ref{theorem:DistPolicyGrad}) that computes the gradient of the probability measure of the cumulative cost, which is a vector of \textit{signed measures}. Compared to the classical policy gradient theorem that is risk-neutral and real-valued, this theorem generalizes it to the measure space and provides richer information for \textit{distributional policy gradient}. Therefore, it can be used to construct the policy gradient for a broad range of coherent risk measures. Based on this, we design a general risk-sensitive distributional policy gradient algorithm (Algorithm~\ref{algo:DPG}), which can be combined with any policy evaluation methods.
    \item For practical use, we develop a categorical approximation in Section~\ref{section:CDPG} that projects any distribution back onto a categorical family with a fixed set of support points. We establish a finite-support optimality guarantee in Lemma~\ref{lemma:OptimGap} to assess the approximation quality. Then, we present a categorical policy gradient theorem in Theorem~\ref{theorem:CatPolicyGrad} and a CDPG algorithm in Algorithm~\ref{algorithm:CDPG}. We mainly focus on CVaR for the simplicity of presentation and discuss extensions to other coherent risk measures under categorical approximation in Remark~\ref{remark:Categorical-other}. Furthermore, in Theorem~\ref{theorem:CDPG_Convergence}, we show that the CDPG algorithm can find an $\epsilon$-stationary point to the categorical approximation problem in finite time under inexact policy evaluation.
    \item In Section~\ref{section:Numerical}, we conduct experiments in a stochastic Cliffwalk environment. Compared with the sample-based policy gradient baseline (SPG) in \cite{tamar2015policy}, the proposed CDPG finds optimal policies using significantly fewer iterations and samples. Furthermore, by comparing risk-averse and risk-neutral policies, we demonstrate the importance of incorporating risk sensitivity into DRL.
\end{itemize}

\section{Preliminaries}
\label{section:Preliminaries}

\subsection{Markov Decision Processes (MDPs)}
Consider a discounted infinite-horizon MDP $\mathcal{M} = (\mathcal{S}, \mathcal{A}, P, C, \gamma)$, 
where $\mathcal{S}$ is a finite set of states, $\mathcal{A}$ is a finite set of actions, 
$P: \mathcal{S}\times\mathcal{A}\rightarrow \Delta(\mathcal{S})$ is the transition kernel, 
$C(s,a)$ is a deterministic immediate cost within $[c_{\min}, c_{\max}]$ \footnote{Our results readily extend to stochastic immediate costs.},
and $\gamma \in [0,1)$ is the discount factor. 
Here, $\Delta(\mathcal{S})$ denotes the probability simplex over $\mathcal{S}$. We aim to find a stationary Markov policy $\pi_\theta: \mathcal{S} \rightarrow \Delta(\mathcal{A})$ which is parameterized by $\theta\in\Theta$. For a fixed policy $\pi_{\theta}$, define $Z_\theta^s$ and $Z_\theta^{(s,a)}$ as the random variables representing the sum of discounted costs along the trajectory following policy $\pi_{\theta}$ starting from state $s$ and starting from $(s,a)$, respectively, i.e., $Z_\theta^s := \sum_{t=0}^{\infty}\gamma^tC(s_t,a_t),\ s_t\sim P(\cdot|s_{t-1},a_{t-1}),\ a_t\sim \pi_{\theta}(\cdot|s_t),\ s_0=s$ and $Z_\theta^{(s,a)} := \sum_{t=0}^{\infty}\gamma^tC(s_t,a_t),\ s_t\sim P(\cdot|s_{t-1},a_{t-1}),\ a_t\sim \pi_{\theta}(\cdot|s_t),\ s_0=s,\ a_0=a$.
The random variable $Z_{\theta}^s: \Omega \to [z_{\text{min}}, z_{\text{max}}]$ is defined on the probability space $(\Omega, \calF, \eta_\theta^s)$, where $\Omega$ is a compact set of outcomes, $\calF$ is the associated $\sigma$-algebra, and $\eta_{\theta}^s$ is the probability measure induced by $Z_\theta^s$, respectively. Similarly, $Z_{\theta}^{(s,a)}: \Omega \to [z_{\text{min}}, z_{\text{max}}]$ is the random variable defined on the probability space $(\Omega, \calF, \eta_\theta^{(s,a)})$, where $\eta_{\theta}^{(s,a)}$ is the probability measure induced by $Z_{\theta}^{(s,a)}$. Denote $\mathcal{Z}$ as the space of all such random variables, $\calP(\bbR)$ as the space of all probability measures over $\bbR$, and $\calM(\bbR)$ as the space of all signed measures over $\bbR$. For any random variable $Z\in\calZ$, we denote $f_{Z}$ and $F_{Z}$ as the corresponding probability density function and cumulative distribution function, respectively. \textit{Throughout the sequel, we omit the dependence of $\pi$ on $\theta$ or replace $\theta$ with $\pi$ whenever it does not cause confusion.}

In risk-neutral RL, we aim to minimize the expected discounted cost:
\allowdisplaybreaks
\begin{align*}
  \min_{\theta\in\Theta}V_\theta(s) :=&\ \mathbb{E}_{\pi,P}\bigl[Z_{\theta}^s\bigr]
  = \mathbb{E}_{\pi,P}\!\Bigl[\sum_{t=0}^{\infty}\gamma^t\,C(s_t, a_t)
      \;\Big|\; s_0=s\Bigr],\\
      &s_t\sim P(\cdot|s_{t-1},a_{t-1}),\ a_t\sim \pi_{\theta}(\cdot|s_t),\ s_0=s,
\end{align*}
where $V_{\theta}(s)$ is the \textit{value function}. A straightforward approach is to update the policy parameter $\theta$ in the gradient descent direction: $\theta\leftarrow \theta - \delta\nabla_\theta V_\theta(s)$, where $\delta$ is the learning rate (step size). A key theoretical tool underpinning this approach is the \emph{policy gradient theorem}~\cite{sutton1999policy}, which provides an explicit formula for $\nabla_\theta V_\theta(s)$:
$\nabla_\theta V_\theta(s) =
\sum_{x} d_{\pi}^{s}(x)\sum_{a}\nabla_\theta \pi(a | x)Q_{\theta}(x, a)$,
where $d^{s}_\pi(x) = \sum_{t=0}^{\infty}\gamma^t \Pr(s_t=x | s_0=s,\pi)$
is the discounted state-visitation distribution, and $Q_\theta(s, a) = \mathbb{E}_{\pi,P}\bigl[Z_{\theta}^{(s,a)}\bigr] =
\mathbb{E}_{\pi,P}\Bigl[\sum_{t=0}^{\infty}\gamma^tC(s_t, a_t)\Big|s_0=s,a_0=a\Bigr]$ is the state-action value function (\textit{Q-function}).

It is well known that, in risk-neutral RL, there always exists an optimal stationary Markov policy. As a result, instead of searching over the whole policy space, we may restrict our attention to the class of stationary Markov policies without loss of optimality. For risk-sensitive RL, \cite{pires2025optimizing} proposed augmenting the state space with a statistic called \textit{stock}, a framework first introduced in \cite{bauerle2011markov}, and identified conditions under which an optimal stationary Markov policy exists for a class of risk measures.
Following this line of work\cite{tamar2015policy,tamar2015optimizing,tamar2016sequential,pires2025optimizing}, we aim to find optimal stationary Markov policies for the risk-sensitive RL problem \eqref{eq:ObjectiveFunction} by proposing distributional policy gradient methods.
We begin by reviewing the definition of coherent risk measures next.

\subsection{Coherent Risk Measures}
A risk measure $\rho: \mathcal{Z} \rightarrow \mathbb{R}$ is called \textit{coherent} if it satisfies the following properties for all $X,Y\in\mathcal{Z}$ and $\lambda \in [0, 1]$~\cite{artzner1999coherent}:
\begin{itemize}
    \item Convexity: $\rho(\lambda X + (1-\lambda)Y) \leq \lambda\rho(X)+(1-\lambda)\rho(Y)$
    \item Translation Invariance: $\rho(X + a) = \rho(X) + a,\ \forall a \in \bbR$
    \item Positive Homogeneity: If $a \geq 0$, then $\rho(a X) = a\rho(X)$
    \item Monotonicity: If $X \preceq Y$, then $\rho(X) \leq \rho(Y)$
\end{itemize}
where $X \preceq Y$ iff $X(\omega) \leq Y(\omega)$ for almost all $\omega \in \Omega$.

A widely used coherent risk measure is CVaR \cite{ruszczynski2006conditional}, which quantifies the expected loss in the tail of a distribution. We briefly review its definition below.

\textbf{Conditional Value-at-Risk (CVaR).} 
Given a risk level $\alpha \in [0, 1]$ (close to 0), the CVaR of a random variable $Z_{\theta}$ is defined as the upper $\alpha$-tail expectation \cite{ruszczynski2006conditional}:
$\rho_{\text{CVaR}}(Z_{\theta}) =\inf_{t\in\bbR}\left\{t+\frac{1}{\alpha}\bbE[(Z_{\theta}-t)_+]\right\}$,
where $(a)_+:=\max\{a,0\}$ and $t\in\mathbb{R}$ is an auxiliary variable to learn the tail distribution. At optimality, $t^*=q_{\alpha}^{\theta}:=\inf\{t:\mathbb{P}(Z_{\theta}\le t)\ge 1-\alpha\}$, which achieves the $(1-\alpha)$-quantile of $Z_{\theta}$. We also call $q_{\alpha}^{\theta}$ the VaR (i.e., $q_{\alpha}^{\theta}=\textrm{VaR}_{1-\alpha}$). 
According to \cite{tamar2015policy}, the gradient of CVaR can be written as
\begin{align}
&\nabla_{\theta}\rho_{\text{CVaR}}(Z_\theta) \nonumber\\
&= \frac{1}{\alpha}\int_{\Omega}\frac{\partial}{\partial\theta } f_{Z_{\theta}}(\omega,\theta)\big(Z_{\theta}(\omega) - q_\alpha^{\theta}\big)\cdot\mathbf{1}_{\{Z_{\theta}(\omega) > q_\alpha^{\theta}\}}d\omega\label{eq:gradient-CVaR}
\end{align}

\begin{remark}
    When $\alpha=1$, CVaR reduces to expectation and Eq.~\eqref{eq:gradient-CVaR} becomes
    \begin{align}
\nabla_{\theta}\mathbb{E}[Z_\theta] 
&= \int_{\Omega}\frac{\partial}{\partial\theta } f_{Z_{\theta}}(\omega,\theta)Z_{\theta}(\omega) d\omega \label{eq:gradient-expectation}
\end{align} 
\end{remark}

Note that in both Eq.~\eqref{eq:gradient-CVaR} and Eq.~\eqref{eq:gradient-expectation}, a key challenge is computing 
$\frac{\partial}{\partial\theta } f_{Z_{\theta}}(\omega,\theta)$ since the probability density function $f_{Z_{\theta}}(\omega,\theta)$ implicitly depends on the policy parameter $\theta$. To address this challenge, we will leverage the DRL framework introduced next to derive this gradient.

\subsection{Distributional Reinforcement Learning (DRL)} 
Rather than learning only the expected value of the cumulative cost, DRL aims to learn the full distribution of the random variable $Z^s$ and $Z^{(s, a)}$ directly. To enable the analysis of the DRL framework, we first define the \emph{pushforward operator} on the space of signed measures $\calM(\bbR)$ below.
\begin{definition}[Pushforward Measure]
\label{def:pushforward}
    Let $\nu\in\calM(\bbR)$ and $f:\bbR\to\bbR$ be a measurable function. The \emph{pushforward measure} $f_{\#}\nu\in\calM(\bbR)$ is defined by $f_{\#}\nu(A):=\nu(f^{-1}(A))$ for all Borel sets $A\subset\bbR$.
\end{definition}
\smallskip
This pushforward operator shifts the support of measure $\nu$ according to the map $f$. In this paper, we focus on the \emph{bootstrap function} $b_{c, \gamma}: \bbR\to\bbR$ defined by $b_{c,\gamma}(z)=c + \gamma z$. Using the pushforward operator together with this bootstrap function, we define the \emph{distributional Bellman operator} $\calT^{\pi}:\calP(\bbR)^{\calS\times\calA}\to \calP(\bbR)^{\calS\times\calA}$ under a fixed policy $\pi$ as follows.
\begin{definition}[Distributional Bellman Operator \cite{rowland2018analysis}]
\label{def:Dist_Bellman_Operator}
Let $\eta \in \calP(\bbR)^{\calS\times\calA}$ be any probability measure. Then the distributional Bellman operator is given by
\small{
\begin{align*}
    (\calT^\pi\eta)^{(s, a)} := \sum_{s^{\prime}\in\calS}P(s'|s, a)  &\sum_{a'\in\calA}\pi(a'|s')(b_{C(s,a), \gamma})_{\#}\eta^{(s^\prime, a^\prime)}.
\end{align*}}
\end{definition}

To analyze the convergence properties of the distributional Bellman operator, we need an appropriate metric on the space of probability measures. The following Wasserstein distance provides such a metric.
\begin{definition}[$p$-Wasserstein Distance]
\label{def:Wasserstein}
 For all $p\ge 1$, the $p$-Wasserstein distance $W_p$ is defined as
\begin{align*}
W_p(\nu_1, \nu_2) = \bigg(\inf_{\lambda \in \Lambda(\nu_1, \nu_2)}\int_{\bbR^2}||x-y||^p\lambda(dx, dy)\bigg)^{1/p}
\end{align*}
for all $\nu_1, \nu_2 \in \calP(\bbR)$, where $\Lambda(\nu_1, \nu_2)$ is the set of joint probability distributions on $\bbR^2$ with marginals $\nu_1$ and $\nu_2$. Then, the maximal form of the $p$-Wasserstein metric $\bar{d}_p$ is defined as $\bar{d}_p(\eta, \nu) = \sup_{(s, a) \in \calS\times\calA} W_p\big(\eta^{(s, a)}, \nu^{(s, a)}\big)$
for all $\eta,\ \nu \in \calP(\bbR)^{\calS \times \calA}$.
\end{definition}

\begin{proposition}[\cite{bellemare2017distributional}]
\label{prop:DistBellmanOperatorContractionMapping}
    The distributional Bellman operator $\calT^\pi$ is a $\gamma$-contraction mapping in $\bar{d}_p$ for all $p\ge 1$, i.e., $\bar{d}_p(\calT^\pi\eta, \calT^\pi\nu)\le \gamma \bar{d}_p(\eta, \nu)$. Furthermore, there is a unique distribution $\eta_{\theta}\in\calP(\bbR)^{\calS\times\calA}$ such that for any initial probability measure $\eta\in\calP(\bbR)^{\calS\times\calA}$, we have
$(\calT^{\pi})^m\eta\to \eta_{\theta},\ \text{as}\ m\to\infty$.
\end{proposition}

Similar to classical RL, we have an analogous \emph{distributional Bellman equation} that characterizes the probability measures $\eta_{\theta}$ as follows. 
\begin{lemma}[Distributional Bellman Equation~\cite{rowland2018analysis}]
\label{lemma:DistBellmanEq}
    For each state $s \in \calS$ and action $a\in\calA$, let $\eta_{\theta}^s$ and $\eta_{\theta}^{(s,a)}$ be the probability measures associated with the random variables $Z_{\theta}^s$ and $Z_{\theta}^{(s,a)}$. Then
    \begin{align*}
    \eta_{\theta}^{(s,a)} &= \sum_{s^{\prime}\in\calS}P(s'|s, a) \sum_{a'\in\calA}\pi_\theta(a'|s')(b_{C(s, a), \gamma})_{\#}\eta_{\theta}^{(s^\prime, a^\prime)}\\
    &=\sum_{s^{\prime}\in\calS}P(s'|s, a) (b_{C(s, a),\gamma})_{\#}\eta_{\theta}^{s^{\prime}}.
    \end{align*}
\end{lemma}

Intuitively, the distributional Bellman operator $\calT^{\pi}$ can be used to evaluate a policy $\pi$, whose fixed point is the cost distribution's probability measure $\eta_{\theta}^s$. Once we obtain the probability measure $\eta_{\theta}^s$, it turns out that the gradient of the probability density function $\frac{\partial}{\partial\theta } f_{Z_{\theta}}(\omega,\theta)$ (the ingredient needed in Eq.\eqref{eq:gradient-CVaR}) is precisely the density of the gradient $\nabla_{\theta}\eta_{\theta}^s$, which is itself a signed measure. This density can then be substituted into Eq.~\eqref{eq:gradient-CVaR} to obtain the policy gradient. We formalize this idea in the following section and Corollary \ref{corollary:MeasureToPDF}.

\section{A New Distributional Policy Gradient Framework}
\label{section:DPG}
In this section, we introduce a general risk-sensitive distributional policy gradient (DPG) framework, as shown in Algorithm~\ref{algo:DPG}.  It consists of the following two main steps:
\begin{itemize}
    \item \textbf{Distributional policy evaluation:} 
    Given a policy $\pi_{\theta}$, for all $(s, a) \in \calS \times \calA$, we evaluate the probability measure of the state-action cost distribution $\eta_{\theta}^{(s,a)} \in \mathcal{P}(\mathbb{R})$ by leveraging the contraction mapping property in Proposition~\ref{prop:DistBellmanOperatorContractionMapping}. Then, the probability measure of the corresponding state cost distribution can be constructed by $\eta_{\theta}^{s} 
        = \sum_{a \in \mathcal{A}} \pi_{\theta}(a|s)\cdot\eta_{\theta}^{(s,a)}$.
    \item \textbf{Distributional policy improvement:} We first obtain the gradient of the probability measure $\nabla_{\theta}\eta_{\theta}^{s}$ using Theorem \ref{theorem:DistPolicyGrad} and compute the policy gradient $\nabla_{\theta}\rho_{\text{CVaR}}\bigl(Z_{\theta}^s\bigr)$ based on $\nabla_{\theta}\eta_{\theta}^{s}$ using Corollary \ref{corollary:MeasureToPDF}. The policy parameter $\theta$ is then updated in the gradient descent direction:  $\theta_{t+1} \leftarrow \theta_t - \delta\cdot\nabla_\theta\rho_{\text{CVaR}}(Z_{\theta_t}^s)$.
\end{itemize}
In this section, we consider an ideal setting in which both the \textit{exact} policy evaluation and the \textit{exact} policy gradient (PG) can be obtained, under \textit{continuous} probability measures. We will consider a practical algorithm to approximate the distribution using a categorical family in Section \ref{section:CDPG}. In Section \ref{section:CDPG}, we will also analyze the iteration complexities of the proposed algorithm under inexact policy evaluation and inexact policy gradient.

\begin{algorithm}[ht]
\caption{\textbf{Distributional Policy Gradient Framework}}
\label{algo:DPG}
\begin{algorithmic}
\REQUIRE Initial Parameter $\theta_1$, Stepsize $\delta$
\FOR{$t = 1, \dots, T$}
    \IF{$\|\nabla_\theta\rho_{\text{CVaR}}(Z_{\theta_t}^s)\| < \epsilon$}
        \STATE Return $\theta_t$
    \ENDIF
     \STATE \textcolor{gray}{\# Distributional Policy Evaluation}
    \WHILE{\textit{not converged}}
        \STATE $\eta_{\theta_t} \leftarrow \mathcal{T}^{\theta_t}\eta_{\theta_t}$ 
    \ENDWHILE
    \STATE \textcolor{gray}{\# Distributional Policy Improvement}
    \STATE Compute policy gradient $\nabla_\theta\rho_{\text{CVaR}}(Z_{\theta_t}^s)$ based on $\nabla_{\theta}\eta_{\theta_t}^{s}$.
    \STATE Update $\theta_{t+1} \leftarrow \theta_t - \delta\cdot\nabla_\theta\rho_{\text{CVaR}}(Z_{\theta_t}^s)$. 
\ENDFOR
    \STATE Return $\theta_T$.
\end{algorithmic}
\end{algorithm}

Before establishing the main results, let us first state the following assumption, which is a standard regularity assumption in the policy gradient literature.

\begin{assumption}\label{assump:bounded-policy}
    Assume that $\pi_{\theta}(a|s)$ is continuously differentiable in $\theta$ and $\nabla_{\theta}\pi_{\theta}(a|s)$ is bounded for all $s\in\mathcal{S}$ and $a\in\mathcal{A}$.
\end{assumption}

The following theorem provides an explicit expression for $\nabla_{\theta}\eta_{\theta}^{s}$, which enables the computation of $\nabla_{\theta}\rho_{\text{CVaR}}(Z_{\theta}^{s})$, in particular the term $\frac{\partial}{\partial\theta} f_{Z_{\theta}}(\omega,\theta)$ in Eq.~\eqref{eq:gradient-CVaR}. All the omitted proofs are presented in Appendix II.

\begin{theorem}[Distributional Policy Gradient Theorem]
\label{theorem:DistPolicyGrad}
Suppose Assumption \ref{assump:bounded-policy} holds.
Let $\eta_\theta \in \calP(\bbR)^{\calS\times\calA}$ denote the fixed point of $\calT^{\pi_{\theta}}$ in Proposition \ref{prop:DistBellmanOperatorContractionMapping}. Let $\tau_{\theta}$ be a trajectory generated by policy $\pi_{\theta}$ starting from $s_0=s$. For any $t\ge 1$, let $\tau_{\theta}(s_0, s_t) := (s_0, a_0, c_0, \dots, s_{t-1}, a_{t-1}, c_{t-1}, s_t)$ be a $t$-step sub-trajectory of $\tau_{\theta}$ truncated at $s_t$. Then
\begin{align}
\label{eq:ContDistPolicyGrad}
\nabla_\theta \eta_\theta^{s} = \bbE_{\tau_{\theta}}\bigg[g(s_0) + \sum_{t=1}^{\infty}\calB^{\tau_{\theta}(s_0, s_t)}g(s_t)\bigg]
\end{align}
where $g(s):= \sum_{a\in\calA}\nabla_\theta\pi_{\theta}(a|s) \cdot \eta_\theta^{(s, a)}$ is a vector of signed measures, and $\calB^{\tau_{\theta}(s_0, s_t)}$ is the $t$-step pushforward operator defined as $\calB^{\tau_{\theta}(s_0, s_t)}:= (b_{{c_0}, \gamma})_{\#}\dots(b_{c_{t-1}, \gamma})_{\#}= (b_{c_{t-1}+\gamma c_{t-2} + \dots + \gamma^{t-1}c_0, \gamma^t})_{\#}$.
\end{theorem}

It turns out that the gradient of the probability density function $\frac{\partial}{\partial \theta} f_{Z_{\theta}^s}(\omega,\theta)$ is the density of the signed measure $\nabla_\theta \eta_\theta^s$, as shown in the next corollary.

\begin{corollary}
\label{corollary:MeasureToPDF}
Suppose $\nabla_\theta \eta_\theta^s$ is well-defined, and both 
$\frac{\partial}{\partial x}\frac{\partial}{\partial \theta}F_{Z_{\theta}^s}(x,\theta)$
and $\frac{\partial}{\partial \theta} f_{Z_{\theta}^s}(x,\theta)$ 
are continuous. Then, we have $\frac{\partial}{\partial \theta} f_{Z_{\theta}^s}(x,\theta) = 
   \frac{\partial}{\partial x}\nabla_\theta \eta_\theta^s((-\infty, x])$, where $\frac{\partial}{\partial x}\nabla_\theta \eta_\theta^s((-\infty, x])$ represents the density function of the signed measure $\nabla_\theta \eta_\theta^s$.
\end{corollary}

Theorem \ref{theorem:DistPolicyGrad} and Corollary \ref{corollary:MeasureToPDF} provide a principled way to calculate the policy gradient for a broad range of coherent risk measures. However, representing an arbitrary continuous probability distribution requires infinitely many parameters, which is computationally intractable. To address this issue, we will design a categorical distributional policy gradient algorithm (CDPG) in Section~\ref{section:CDPG}, which approximates any distribution using a parametric family of distributions with finite support. In Section~\ref{section:CDPG}, we will also establish the iteration complexity of CDPG for finding $\epsilon$-stationary points for static CVaR.

\section{Categorical Distributional Policy Gradient Algorithm with Provable Convergence}
\label{section:CDPG}
In this section, we focus on a \emph{categorical approximation} problem, which projects any distribution onto a categorical family with a fixed set of support points. 
We first quantify the optimality gap between this categorical approximation and the original problem under a finite number of support points in Section~\ref{subsection:CatApprox}. Different from the distributional dynamic programming algorithms proposed in \cite{bellemare2017distributional,rowland2018analysis,pires2025optimizing}, we derive a categorical distributional policy gradient theorem (Theorem~\ref{theorem:CatPolicyGrad}) and propose the CDPG algorithm in Section~\ref{subsection:CatPolicyGrad}. Under inexact policy evaluation (using finite rounds of policy evaluation or finite samples), we analyze the convergence rate of CDPG to stationary points in Section~\ref{subsection:Inexactness}.

\subsection{Categorical Approximation}
\label{subsection:CatApprox}
We approximate any cost distribution under policy $\pi_{\theta}$ by the following categorical family with $N$ supports:
\begin{align*}
\mathcal{P}_N^{\theta} = \bigg\{\sum_{i=1}^{N}p_i^\theta\delta_{z_i} \mid p_1^\theta,\dots,p_N^\theta \geq 0,\ \sum_{i=1}^{N}p_i^\theta=1\bigg\},
\end{align*}
where the fixed support points $z_{\min} = z_1 < \dots < z_N = z_{\max}$ equally partition $[z_{\min}, z_{\max}]$. Because the Bellman operator $\calT^{\pi}$ shifts the support, it does not preserve the categorical family $\calP^{\theta}_N$. We therefore introduce a projection operator $\Pi_{\mathcal{C}}$ to map distributions back to $\calP^{\theta}_N$~\cite{dabney2018distributional}.

\begin{definition}
\label{def:ProjOperator}
The \textit{projection operator} $\Pi_\mathcal{C}:\mathcal{P}(\mathbb{R})\to\mathcal{P}_N$ is defined by its action on a Dirac measure:
\small{
\begin{align*}
\Pi_\calC(\delta_{y}) =
\begin{cases}
\delta_{z_1}, & \text{if } y \leq z_1 \\
\dfrac{z_{i+1}-y}{z_{i+1}-z_i}\delta_{z_i} + \dfrac{y-z_i}{z_{i+1}-z_i}\delta_{z_{i+1}}, & \text{if } z_i < y \leq z_{i+1}\\
\delta_{z_N}, & \text{if } y > z_N
\end{cases}
\end{align*}}
\end{definition}
This operator extends linearly to any measure in $\mathcal{M}(\mathbb{R})$, such that 
$\Pi_\calC(\sum_{i=1}^N q_i\delta_{z_i})=\sum_{i=1}^Nq_i\Pi_\calC(\delta_{z_i})$. We leverage this projection to define the \textit{projected distributional Bellman operator} $\Pi_{\calC}\calT^{\pi}$ below.

\begin{definition}[\cite{rowland2018analysis}]
\label{def:ProjDistrBellmanOperator}
For any $\eta \in \calP(\bbR)^{\calS\times\calA}$, define
\begin{align*}
&(\Pi_\calC\calT^\pi\eta)^{(s, a)} \\
&= \Pi_\calC\bigg[\sum_{s'}P(s'|s, a) \sum_{a'}\pi(a'|s') \cdot\big(b_{C(s, a), \gamma}\big)_{\#}\eta^{(s', a')}\bigg].
\end{align*}
\end{definition}
The projected distributional Bellman operator $\Pi_\calC\calT^\pi$ fails to be contractive under the Wasserstein distance, but it has been shown to be contractive under the Cramér distance, which we define below.
\begin{definition}[Cram\'er distance]
\label{def:CramerDistance}
For distributions $\nu_1, \nu_2 \in \mathcal{P}(\mathbb{R})$ with cumulative distribution functions $F_{\nu_1}$ and $F_{\nu_2}$, the Cramér distance $\ell_2$ is defined by:
\begin{align*}
\ell_2(\nu_1, \nu_2) = \left(\int_{\mathbb{R}} (F_{\nu_1}(x) - F_{\nu_2}(x))^2 \, dx\right)^{1/2}.
\end{align*}
\end{definition}

\begin{proposition}[\cite{rowland2018analysis}]
\label{prop:ProjDistBellmanOperatorContractionMapping}
The projected distributional Bellman operator $\Pi_\calC\calT^\pi$ is a $\sqrt{\gamma}$-contraction mapping under the supremum-Cramér distance $\bar{l}_2$, where 
$$\bar{\ell}_2(\nu_1, \nu_2) = \sup_{(s, a) \in \calS \times \calA} \ell_2(\nu_1(s, a), \nu_2(s, a)).$$
\end{proposition}

\begin{lemma}[\cite{rowland2018analysis}]
\label{lemma:ProjDistBellmanEquation}
Let $\eta_{N, \infty} \in \calP_N^{\calS \times \calA}$ be the fixed point of $\Pi_{\calC}\calT^{\pi}$. Then, for any $s \in \calS$ and $a \in \calA$,
\begin{align*}
    \eta_{N, \infty}^{(s,a)} = \sum_{s'} P(s'|s, a)\Pi_\calC \bigl(b_{C(s,a),\gamma}\bigr)_\# \eta_{N, \infty}^{s'}.
\end{align*}
\end{lemma}
Consequently, repeatedly applying $\Pi_\calC\calT^\pi$ converges to its unique fixed point
$\eta_{N, \infty}\in\mathcal{P}_N^{\mathcal{S}\times\mathcal{A}}$. Based on this fixed point, we next focus on the following \emph{categorical approximation} problem:
\begin{align}
\label{eq:CatApproxProblem}
\min_{\theta}\rho_{\text{CVaR}}\bigl(Z_{N,\infty}^s\bigr),
\end{align}
where $Z_{N,\infty}^s \sim \eta_{N, \infty}^{s}:=\sum_{a \in \calA}\pi_{\theta}(a|s)\cdot\eta_{N, \infty}^{(s, a)} \in \calP_N$.

This raises a natural question: how close is the optimal objective value of the categorical approximation problem \eqref{eq:CatApproxProblem} to that of the original problem \eqref{eq:ObjectiveFunction}? More specifically, how should we choose $N$ to ensure that the approximation error is at most $\epsilon_{opt}$? The next lemma quantifies the approximation quality of the categorical approximation problem \eqref{eq:CatApproxProblem} and provides such a bound on $N$.
\begin{lemma}[Finite-Support Optimality Guarantee]
\label{lemma:OptimGap}
For any $\epsilon_{opt} > 0$, we have
$|\min\limits_{\theta}\rho_{\text{CVaR}}(Z^s) - \min\limits_{\theta}\rho_{\text{CVaR}}(Z_{N,\infty}^s)| \leq \epsilon_{opt}$,
whenever 
\begin{align*}
N \geq \frac{(z_{\max}-z_{\min})^2}{\alpha^2(1-\gamma) \epsilon_{opt}^2}+1.
\end{align*}
\end{lemma}
According to Lemma \ref{lemma:OptimGap}, the required number of support points $N$ is in the order of $\calO\left(\frac{1}{(1-\gamma)\epsilon_{opt}^2}\right)$. As $\epsilon_{opt}\to 0$ or $\gamma\to 1$, the required number of support points $N$ tends to infinity.

\subsection{CDPG Algorithm}
\label{subsection:CatPolicyGrad}
To introduce our CDPG algorithm for the categorical approximation problem \eqref{eq:CatApproxProblem}, we first derive the \emph{categorical policy gradient theorem}, which is a tailored version of Theorem~\ref{theorem:DistPolicyGrad}.
\begin{theorem}[Categorical Policy Gradient Theorem]
\label{theorem:CatPolicyGrad}
Let $\eta_{N, \infty} \in \calP_N^{\calS\times\calA}$ denote the fixed point of $\Pi_{\calC}\calT^{\pi}$. Let $\tau_{\theta}$ be a trajectory generated by policy $\pi_\theta$ starting from $s_0=s$. For any $t \ge 1$, let $\tau_{\theta}(s_0,s_t)$ be the $t$-step sub-trajectory truncated at $s_t$. Then
\begin{align}
\label{eq:CarPolicyGradient}
\nabla_\theta \eta_{N, \infty}^{s} 
= \mathbb{E}_{\tau_{\theta}}\bigg[g_{N,\infty}(s_0) + \sum_{t=1}^{\infty}\tilde{\mathcal{B}}^{\tau_{\theta}(s_0,s_t)}g_{N,\infty}(s_t)\bigg],
\end{align}
where $g_{N,\infty}(s):=\sum_{a \in \calA}\nabla_\theta\pi_\theta(a|s)\eta_{N, \infty}^{(s,a)}$ is a vector of signed measures, and $\tilde{\mathcal{B}}^{\tau_{\theta}(s_0,s_t)}$ is the $t$-step projected pushforward operator defined by $\tilde{\calB}^{\tau_{\theta}(s_0, s_t)} := \Pi_\calC(b_{c_{0}, \gamma})_{\#}\Pi_\calC(b_{c_{1}, \gamma})_{\#}\dots\Pi_\calC(b_{c_{t-1}, \gamma})_{\#}$.
\end{theorem}

\begin{remark}(\textit{Categorical Policy Gradient for CVaR}):
\label{remark:CatPolicyGradComputation}
For any categorical distribution $\eta_{N,\infty}^{s}=\sum_{i=1}^{N} p_i^{\theta} \,\delta_{z_i} \in \mathcal{P}_N$, we have $\nabla_\theta \eta_{N,\infty}^{s}
= \nabla_{\theta}\Bigl(\sum_{i=1}^{N} p_i^{\theta} \,\delta_{z_i}\Bigr)
= \sum_{i=1}^N \nabla_\theta p_i^{\theta}\delta_{z_i}$
and as a result, $\nabla_\theta \eta_{N,\infty}^{s}(\{z_i\}) = \nabla_{\theta}p_i^{\theta}$ for $i = 1, \dots, N$.
Theorem \ref{theorem:CatPolicyGrad} provides $\nabla_\theta \eta_{N,\infty}^{s}$ and thus $\nabla_\theta p_i^{\theta}$ for all $i=1,\ldots, N$, which can be used to compute the policy gradient for CVaR. Here, for categorical distributions, the probability density function $f_{Z_{\theta}}(\omega,\theta)$ is replaced with the probability mass $p_i^{\theta}$, and correspondingly, we have
\begin{align}
\label{eq:CVaR_Grad}
    \nabla_{\theta}\rho_{\text{CVaR}}(Z_{N,\infty}^s) = \frac{1}{\alpha}
    \sum_{i=1}^{N}\nabla_{\theta}p^{\theta}_i\bigl(z_i - q_\alpha^{\theta}\bigr) \mathbf{1}_{\{z_i > q_\alpha^{\theta}\}},
\end{align}
where $q_\alpha^{\theta}$ is the $(1-\alpha)$-quantile of $Z_{N,\infty}^s$, and $\nabla_{\theta}p^{\theta}_i$ can be obtained from $\nabla_\theta \eta_{N,\infty}^{s}$ (Theorem \ref{theorem:CatPolicyGrad}).
\end{remark}

Next, we show that the categorical policy gradient presented in Theorem~\ref{theorem:CatPolicyGrad} can be also useful for other coherent risk measures.
\begin{remark}(\textit{Categorical Policy Gradient for Other Coherent Risk Measures}):\label{remark:Categorical-other}
According to Theorem 4.2 in \cite{tamar2015policy}, for general coherent risk measures $\rho$, under mild conditions, we have
\small
    \begin{align}
&\nabla_{\theta}\rho(Z_{N,\infty}^s) = \sum_{i=1}^N\xi_{\theta}^*(i)\nabla_{\theta}p_i^{\theta}(z_i - \lambda_{\theta}^{\ast, P})\nonumber\\
&- \sum_{e \in \calE}\lambda_{\theta}^{\ast, \calE}(e)\nabla_{\theta}g_e(\xi_{\theta}^{*}, p^{\theta}) - \sum_{j \in \calJ}\lambda_{\theta}^{\ast, \calJ}(j)\nabla_{\theta}h_j(\xi_{\theta}^{*}, p^{\theta})\label{eq:gradient-coherent}
\end{align}
\normalsize
where $(\xi_{\theta}^\ast, \lambda_{\theta}^{\ast, P}, \lambda_{\theta}^{\ast, \calE}, \lambda_{\theta}^{\ast, \calJ})$ is a saddle point of a Lagrangian function and $g_e, h_j$ are the functions in defining the risk envelope of a coherent risk measure (see Assumption 2.2 in \cite{tamar2015policy}). When these functions and the saddle point admit closed forms, one may construct the policy gradient using Eq.~\eqref{eq:gradient-coherent} in an analytical form, where $\nabla_{\theta} p_i^{\theta}$ can be obtained from Theorem~\ref{theorem:CatPolicyGrad} and Remark \ref{remark:CatPolicyGradComputation}. Besides CVaR, another widely used coherent risk measure is the mean-semideviation, defined as follows:
\begin{align*}
\rho_{\textit{MSD}}(Z_\theta; \alpha) = \bbE[Z_\theta] + \alpha\bigg(\bbE\big[(Z_\theta - \bbE[Z_\theta])_+^2\big]\bigg)^{1/2}
\end{align*}
with $\alpha\in[0,1]$ representing the weight of the semideviation.
It admits the following closed-form gradient (see, e.g., Proposition 4.3 in \cite{tamar2015policy}):
\begin{align*}
&\nabla_\theta\rho_{\textit{MSD}}(Z_\theta; \alpha) = \nabla_{\theta}\bbE[Z_{\theta}]+\\
&\frac{\alpha\bbE\left[(Z_{\theta}-\bbE[Z_{\theta}])_+(\nabla_\theta \log p^{\theta}(Z_{\theta}-\bbE[Z_{\theta}])-\nabla_\theta\bbE[Z_{\theta}])\right]}{\mathbb{SD}(Z_{\theta})}.
\end{align*}
With a categorical distribution $Z_{\theta}\in\mathcal{P}_N$, we have $\nabla_{\theta}\bbE[Z_{\theta}]=\sum_{i=1}^N\nabla_\theta p_i^{\theta} z_i$, where $\nabla_\theta p_i^{\theta}$ can be obtained from Theorem~\ref{theorem:CatPolicyGrad}. Therefore, one may construct the policy gradient for the mean-semideviation using the categorical policy gradient theorem.
\end{remark}

To facilitate the theoretical analysis, we next consider infinite-horizon discounted MDPs with an absorbing terminal state. Such absorbing states naturally arise in many real-world tasks, where the subsequent costs are set to 0 at termination (e.g., the goal state in the Cliffwalk environment considered in Section \ref{section:Numerical}). We assume that the absorbing state is reached almost surely in finite time and that the corresponding termination time has a finite expectation uniformly over $\theta\in\Theta$. In this case,  $g_{N,\infty}(s_{\text{abs}})=0$ for any absorbing state $s_{\text{abs}}$, and all terms after the trajectory reaches $s_{\text{abs}}$ vanish. Consequently, the gradient in Eq. \eqref{eq:CarPolicyGradient} can be equivalently recast as 
\begin{align}
\nabla_\theta \eta_{N, \infty}^{s} 
= \mathbb{E}_{\tau_{\theta}}\left[g_{N,\infty}(s_0) + \sum_{t=1}^{|\tau_{\theta}|}\tilde{\mathcal{B}}^{\tau_{\theta}(s_0,s_t)}g_{N,\infty}(s_t)\right] \label{eq:gradient-eta-finite}
\end{align}
where $\tau_{\theta}= (s_0, a_0, c_0, \dots, s_{|\tau_{\theta}|-1}, a_{|\tau_{\theta}|-1}, c_{|\tau_{\theta}|-1}, s_{|\tau_{\theta}|}=s_{\text{abs}})$ denotes a trajectory terminating at the absorbing state $s_{\text{abs}}$ and $|\tau_{\theta}|$ denotes its length. Hence, under the above assumptions, the gradient in Eq. \eqref{eq:gradient-eta-finite} is well-defined and can be evaluated using finite-length trajectories that terminate upon reaching the absorbing state. Accordingly, we sample finite-length trajectories $\tau_{\theta}$ to approximate this gradient in the following.

We now summarize the main steps of CDPG in  Algorithm~\ref{algorithm:CDPG}. Specifically, we first approximate $\eta_{N, \infty}$ by $\eta_{N,k}$, obtained by applying the operator $\Pi_{\calC}\calT^\pi$ for a finite number of rounds. The required number of rounds $k$ depends on the length of the sampled trajectory $|\tau_{\theta}|$ and the number of supports $N$, as established in Theorem~\ref{theorem:CDPG_Convergence}. Next, we use Theorem~\ref{theorem:CatPolicyGrad} to approximate $\nabla_\theta \eta_{N, \infty}^{s}$ by replacing $\eta_{N, \infty}$ with $\eta_{N,k}$ on the right-hand side of Eq.~\eqref{eq:CarPolicyGradient}. This results in a surrogate for $\nabla_\theta \eta_{N, \infty}^{s} $, denoted by $\nabla_\theta \eta_{N, k}^{s}$. Note that $\nabla_\theta \eta_{N, k}^{s}$ is used solely to denote the approximation obtained by replacing $\eta_{N, \infty}$ with $\eta_{N,k}$ in Eq.~\eqref{eq:CarPolicyGradient}. It should not be interpreted as the true gradient of $\eta_{N, k}^{s}$. 
Based on Remark~\ref{remark:CatPolicyGradComputation}, we then replace all the quantities associated with $\nabla_\theta\eta_{N, \infty}^{s}$ by the ones associated with $\nabla_\theta \eta_{N, k}^{s}$ in Eq.~\eqref{eq:CVaR_Grad} to construct an approximation to $\nabla_{\theta}\rho_{\text{CVaR}}(Z_{N,\infty}^s)$, denoted by $h$. Finally, we update $\theta$ in the direction of $-h$. 

\begin{algorithm}[ht!]
\caption{\textbf{CDPG Algorithm}}
\begin{algorithmic}
\label{algorithm:CDPG}
\REQUIRE initial parameter $\theta_1$, stepsize $\delta$, total epoch $T$, boundary [$z_{\min}, z_{\max}$], support size $N$ 

\FOR{$t = 1, \dots, T$}
    \STATE Sample a trajectory $\tau_{\theta_t}$ following $\pi_{\theta_t}$ 

    \STATE \textcolor{gray}{\# Categorical Distributional Policy Evaluation}

    \STATE Initialize $\eta_{N, 0} \in \calP_N^{\calS \times \calA}$ 
    \STATE $\eta_{N, k} \leftarrow (\Pi_{\calC}\calT^\pi)^k \eta_{N, 0}$

    \STATE \textcolor{gray}{\# Categorical Distributional Policy Improvement}

     \STATE $\nabla_\theta\eta_{N, k}^{s} \leftarrow \sum_{a}\nabla_\theta\pi_{\theta_t}(a|s)\cdot\eta_{N,k}^{(s, a)}$ 

    \FOR{$\tau = 1, \dots, |\tau_{\theta_t}|$} 
        \STATE Compute $g_{N,k}(s_{\tau}) = \sum_{a}\nabla_\theta\pi_{\theta_{\tau}}(a|s_{\tau})\cdot\eta_{N,k}^{(s_{\tau}, a)}$ 
        
        \STATE $\nabla_\theta\eta_{N, k}^{s} \leftarrow \nabla_\theta\eta_{N, k}^{s} + \tilde{\calB}^{\tau_{\theta}(s_0, s_{\tau})}(g_{N,k}(s_{\tau}))$ 
        
    \ENDFOR
    \STATE Denote $\nabla_{\theta}\eta^{s}_{N, k}=\sum_{i=1}^N\tilde{p}_i\delta_{z_i}$ where $\tilde{p}_i:=\nabla_{\theta}\eta^{s}_{N, k}(\{z_i\})$.
    
    \STATE Compute an approximate gradient $h$ based on Remark~\ref{remark:CatPolicyGradComputation}, i.e., $h= \frac{1}{\alpha}
    \sum_{i=1}^{N}\tilde{p}_i\bigl(z_i - \tilde{q}_\alpha^{\theta}\bigr) \mathbf{1}_{\{z_i > \tilde{q}_\alpha^{\theta}\}}$,
where $\tilde{q}_\alpha^{\theta}$ is the $(1-\alpha)$-quantile of $Z_{N,k}^s$.
    
    \STATE Update $\theta_{t+1} \leftarrow \theta_t -\delta\cdot h$.
\ENDFOR
\end{algorithmic}
\end{algorithm}

Note that Algorithm~\ref{algorithm:CDPG} is described in a model-based setting, where we apply the projected Bellman operator $\Pi_{\calC}\calT^\pi$ a finite number of rounds. This is needed for bounding the approximation gap between $\eta_{N,k}$ and the fixed point $\eta_{N,\infty}$, because the operator $\Pi_{\calC}\calT^\pi$ has been shown to be a contraction mapping in Proposition~\ref{prop:ProjDistBellmanOperatorContractionMapping}. However, this step can be also done approximately in a model-free setting, by minimizing the Kullback-Leibler (KL) divergence between the current cost distribution and the next-state distribution after one-step projected Bellman update using sampled transitions. We refer interested readers to Algorithm 1 in \cite{bellemare2017distributional} for detailed descriptions. In our numerical experiments in Section~\ref{section:Numerical}, we use this model-free algorithm. For the ease of analysis, we continue with the model-based setting presented in Algorithm~\ref{algorithm:CDPG}.

\subsection{Finite-Time Local Convergence of CDPG with Inexact Policy Evaluation}
\label{subsection:Inexactness}
Since Model \eqref{eq:CatApproxProblem} is highly non-convex with respect to the policy parameter $\theta$, the best we can hope for the policy gradient algorithm is convergence to a stationary point. Moreover, it is common that the state-action distribution $\eta_{N,\infty}$ is available only in an approximate fashion when it has to be evaluated using finite rounds, e.g., $\eta_{N,k}$ presented in Algorithm~\ref{algorithm:CDPG}. In this section, we analyze Algorithm~\ref{algorithm:CDPG} under \emph{inexact policy evaluation}, where we only conduct a finite round of policy evaluation at each iteration. We establish an iteration complexity required for Algorithm~\ref{algorithm:CDPG} to obtain an $\epsilon$-stationary point to Model \eqref{eq:CatApproxProblem}. 

Let $\eta_{N, \infty}$ be the limiting distribution of $\Pi_{\calC}\calT^\pi$ and let $\eta_{N, k}$ be the categorical distribution obtained after $k$ iterations of applying the operator $\Pi_{\calC}\calT^\pi$, starting from an initial distribution $\eta_{N, 0} \in \mathcal{P}_N$. Denote the categorical measure as $\eta_{N, \infty}=\sum_{i=1}^{N} p_i^{N,\infty}\delta_{z_i}$ and $\eta_{N, k}=\sum_{i=1}^{N} p_i^{N,k}\delta_{z_i}$. Note that we omit the dependence of $\eta_{N, \infty}$, $\eta_{N, k}$, $p_i^{N,\infty}$ and $p_i^{N,k}$ on $\theta$. We first make the following assumptions about the probability masses $p_i^{N,\infty}$ and the $(1-\alpha)$-quantile of $\eta_{N, \infty}$.

\begin{assumption}[Bounded $\nabla_{\theta}p_j^{N,\infty}$ and $\nabla_{\theta}^2p_j^{N,\infty}$]
\label{assumption:bounded-derivative}
    The probability masses have well-defined and bounded first- and second-order derivatives:
$\|\nabla_{\theta}p_j^{N,\infty}\|_\infty \leq C_P^{(1)}, \ \|\nabla_{\theta}^2p_j^{N,\infty}\|_\infty \leq C_P^{(2)}$ for all $j \in [N]$.
\end{assumption}

\begin{assumption}[Uniform quantile separation]
\label{assumption:AlphaQuantile}
For any $\theta\in\Theta$, there exists an index $j(\theta)\in[N]$ such that
$\sum_{i=1}^{j(\theta)-1}p_i^{N,\infty} < 1-\alpha < \sum_{i=1}^{j(\theta)}p_i^{N,\infty}.$
Let $\epsilon_{\alpha}:=\inf_{\theta\in\Theta}\min\{\sum_{i=1}^{j(\theta)}p_i^{N,\infty}- (1-\alpha), (1-\alpha)-\sum_{i=1}^{j(\theta)-1}p_i^{N,\infty}\}$. Suppose that $\epsilon_{\alpha}>0$.
\end{assumption}

Assumption \ref{assumption:bounded-derivative} is commonly seen in the literature to provide smoothness guarantees, see, e.g., in \cite{tamar2015policy,huang2021convergence}. Assumption~\ref{assumption:AlphaQuantile} ensures that the probability level $1-\alpha$ falls strictly inside the jump of the CDF at $z_{j(\theta)}$ (between the cumulative masses up to $z_{j(\theta)-1}$ and $z_{j(\theta)}$). Consequently, the $(1-\alpha)$-quantile $q_{\alpha}^{\theta}$ is uniquely located at some support point $z_{j(\theta)}$ and does not change if one perturbs $\theta$ slightly. If Assumption~\ref{assumption:AlphaQuantile} is not satisfied, one can either perturb $\alpha$ slightly or adjust the support points. These two assumptions guarantee that CVaR is $\beta$-smooth under the discrete setting, as we show in the next lemma.

\begin{lemma}
\label{lemma:CVaRSmooth}
Under Assumptions~\ref{assumption:bounded-derivative}--\ref{assumption:AlphaQuantile}, the objective function~\eqref{eq:CatApproxProblem} is $\beta$-smooth.
\end{lemma}

Next, we derive an iteration complexity bound for the CDPG algorithm (Algorithm~\ref{algorithm:CDPG}) while accounting for inexact gradient estimation due to finite policy evaluation rounds.

\begin{theorem}[CDPG Convergence]
\label{theorem:CDPG_Convergence}
Suppose Assumptions~\ref{assumption:bounded-derivative}-\ref{assumption:AlphaQuantile} hold. In Algorithm~\ref{algorithm:CDPG}, let the stepsize $\delta=1/\beta$ and the number of $\Pi_{\calC}\calT^\pi$ operator evaluations $k(N,|\tau_{\theta}|) = \kappa N(|\tau_{\theta}|+1)$. Denote the optimal objective function value of \eqref{eq:CatApproxProblem} as $\rho_N^*:=\min_{\theta\in\Theta}\rho_{\text{CVaR}}(Z_{\theta, N}^s)$ and the objective value of the initial parameter $\theta_1$ as $\rho_N^1:=\rho_{\text{CVaR}}(Z_{\theta_1, N}^s)$. For any $\epsilon>0$, we have $\min_{t = 1, \dots, T}\|\nabla_{\theta}\rho_{\text{CVaR}}(Z_{\theta_t, N}^s)\|_2^2 \leq \epsilon$, whenever
\begin{align*}
&T \geq \dfrac{4\beta(\rho_{\text{CVaR}}(Z_{\theta_1, N})-\min_{\theta\in\Theta}\rho_{\text{CVaR}}(Z_{\theta, N}))}{\epsilon}\ \text{ and }\\
&\kappa >\max\bigg\{\dfrac{2\log(C_3N^{3/2}\epsilon^{-1/2})}{N\log(1/\gamma)}, \dfrac{\log(C_1N\epsilon_\alpha^{-2})}{N\log(1/\gamma)}\bigg\},
\end{align*}
where $C_1,\ C_3$ are constants that depend on problem parameters.
\end{theorem}

Theorem \ref{theorem:CDPG_Convergence} establishes finite-time convergence to an $\epsilon$-stationary point for arbitrary $\epsilon>0$. As $\gamma\to 1$, $\kappa$ increases because the problem becomes more challenging. Furthermore, the number of policy evaluation rounds required per iteration $k(N,|\tau_{\theta}|)$ increases with $N$ only logarithmically.

\section{Numerical Experiments}
\label{section:Numerical}
In this section, we conduct numerical experiments to evaluate our algorithm in a stochastic Cliffwalk environment. We first introduce our experimental setup in Section \ref{sec:setup}. We then compare the performance of risk-sensitive and risk-neutral policies in the Cliffwalk environment in Section \ref{sec:risk-comparison}. Finally, we compare our CDPG with a non-DRL sample-based policy gradient baseline (SPG) \cite{tamar2015policy} in Section \ref{sec:benchmark}.

\subsection{Experimental Setup}\label{sec:setup}
We evaluate our algorithm in a stochastic $3\times 3$ Cliffwalk environment (see Figure~\ref{fig:cliffwalk}), where the agent needs to navigate from the bottom left [3,1] to the bottom right [3,3] under the risk of falling off the cliff [3,2]. Each step incurs a cost of 10, while falling off the cliff results in a cost of 30 and forces the agent to return to the start state [3,1]. We assume that the state above the cliff [2,2] is slippery, with a probability $p=0.2$ of falling off the cliff when entered. The discount factor is set to $\gamma=0.95$. The shortest path is the red path in Figure \ref{fig:cliffwalk}, which has a cost of 40 (without discounting) if not entering the cliff and incurs a higher cost otherwise; a safer path is the blue path, which induces a deterministic cost of 60 (without discounting). We parameterize the policy using the softmax function $\pi_\theta(a|s) = \tfrac{\exp(\theta_{a, s})}{\sum_{a' \in \mathcal{A}(s)} \exp(\theta_{a', s})}$, where $\{\theta_{a, s}\}_{a\in\mathcal{A},s\in\mathcal{S}}$ are the policy parameters to be optimized. 

\begin{figure}[ht!]
    \centering
    \includegraphics[width=0.14\textwidth]{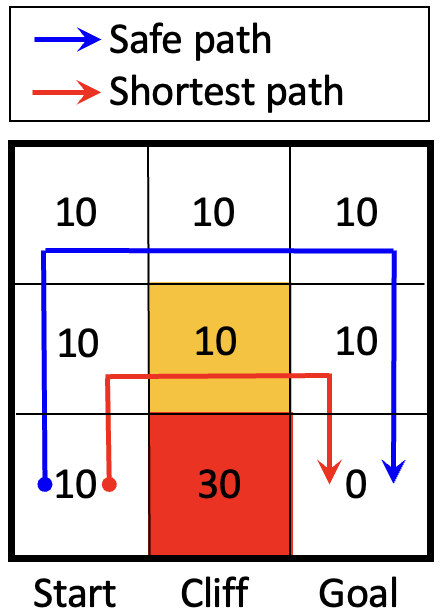}
    \caption{Illustration of the stochastic Cliffwalk environment.}
    \label{fig:cliffwalk}
\end{figure}

For the policy evaluation step in our CDPG algorithm, we adopt the \textit{Online Categorical Temporal-Difference Learning} algorithm (see Algorithm 3.4 in~\cite{bellemare2023distributional}). Since CDPG is computationally demanding, we further incorporate the following two speed-up strategies:
\begin{itemize} 
    \item \textbf{Warm Start:} The next policy evaluation initializes with the previously estimated distribution.
    \item \textbf{Early Stopping:} The policy evaluation stops if the distance between the current and previous distributions does not decrease for several consecutive iterations. 
\end{itemize}

At each training episode, we evaluate the obtained policy by sampling 10 independent trajectories and report the average test cost over the sampled trajectories. We also report the distance to the safe path. We conduct experiments over 10 independent runs and report their average and standard deviation of the performance results. All experiments are conducted on an Intel® Core™ i5-12600K processor and an NVIDIA 4080 Super GPU.

\subsection{Comparison between Risk-Sensitive and Risk-Neutral Policies} \label{sec:risk-comparison}
We first implement the CDPG algorithm to optimize CVaR in the Cliffwalk environment, where a smaller $\alpha$ represents a more risk-averse attitude. We compare the performance under risk-averse ($\alpha=0.1$) and risk-neutral ($\alpha=1$) settings. 

Figure~\ref{figure:Cliffwalking_CDPG}(a) shows that the risk-neutral policy exhibits a high variance in the test cost, while the risk-averse policy chooses the safe path that results in a stable cost. This highlights the importance of risk control in high-stake applications. Additionally, from Figure \ref{figure:Cliffwalking_CDPG}(b), by incorporating the warm-start initialization (Warm) and early stopping (Stop) in the Categorical Distributional Policy Evaluation in Algorithm~\ref{algorithm:CDPG}, the training time can be accelerated by a factor of five compared to the original CDPG algorithm.

\begin{figure}[ht!]
    \centering
\subfigure[Test cost v.s. \# iterations.]{
	\includegraphics[width=0.46\linewidth]{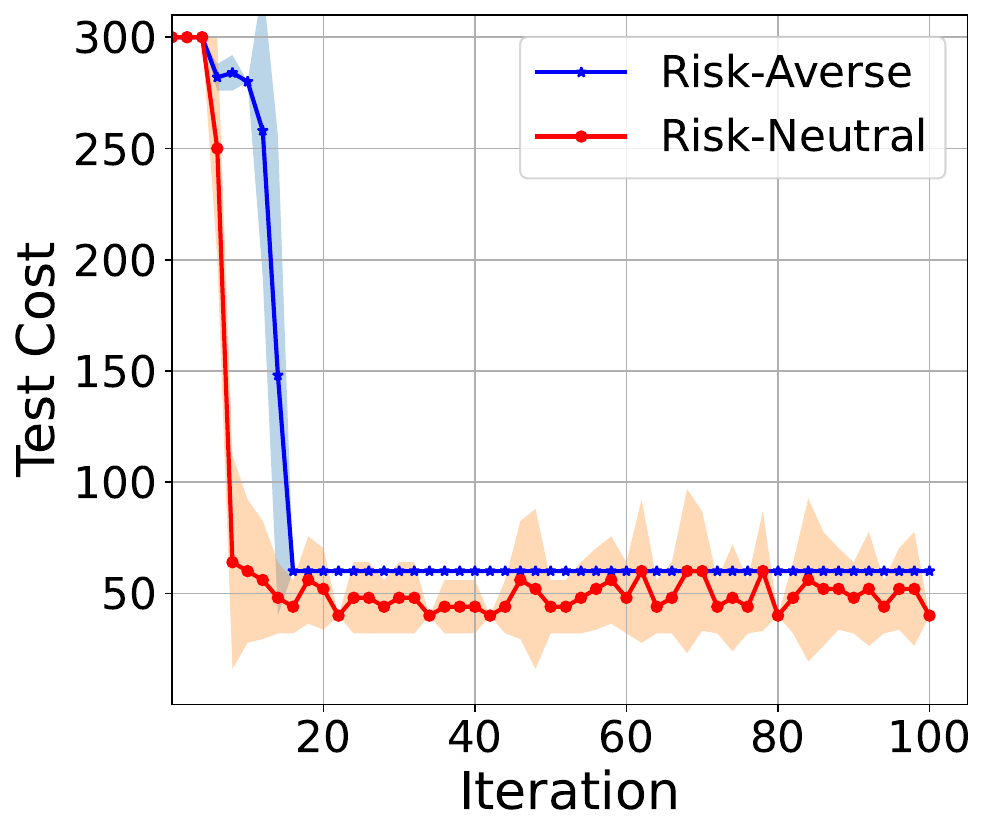}
}
\subfigure[Test cost v.s. time (sec.).]{
	\includegraphics[width=0.46\linewidth]{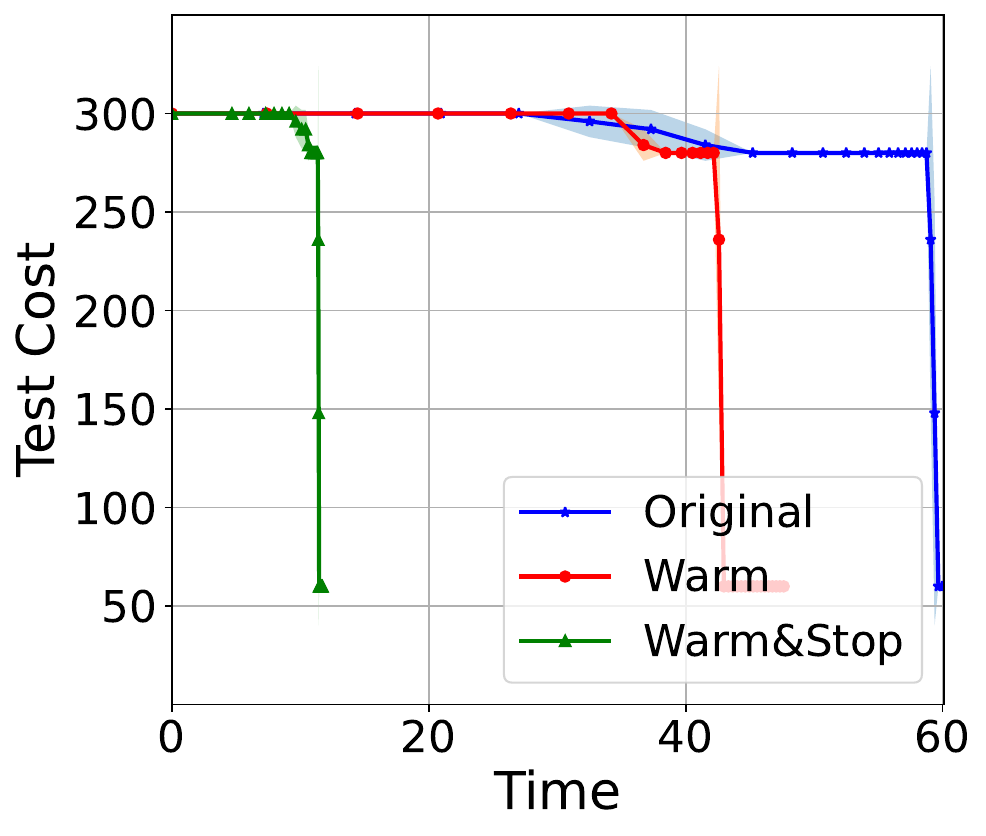}
}
\caption{Comparison between risk-averse and risk-neutral CDPG in the Cliffwalk environment.}
    \label{figure:Cliffwalking_CDPG}
\end{figure}

\subsection{Comparison with SPG on Cliffwalk}\label{sec:benchmark}
To ensure a fair comparison, we compare CDPG against the sample-based policy gradient (SPG) method \cite{tamar2015policy}, which also derives analytical gradients for risk-sensitive objectives but approximates them via Monte Carlo sampling without DRL. 
Figure~\ref{figure:CliffWalking_CDPG_SPG} compares CDPG and SPG on Cliffwalk. Figures~\ref{figure:CliffWalking_CDPG_SPG}(a), (b), and (c) track average test cost v.s. trajectories, iterations, and time, while Figure~\ref{figure:CliffWalking_CDPG_SPG}(d) shows the distance to the optimal policy (safe path) after fixed iterations under various sample sizes, respectively. Despite slightly higher computational time (around 10 seconds), CDPG achieves four times better sample efficiency than SPG (see Figure~\ref{figure:CliffWalking_CDPG_SPG}(a)).

\begin{figure}[ht!]
    \centering
\subfigure[Test cost v.s. \# trajectories.]{
	\includegraphics[width=0.46\linewidth]{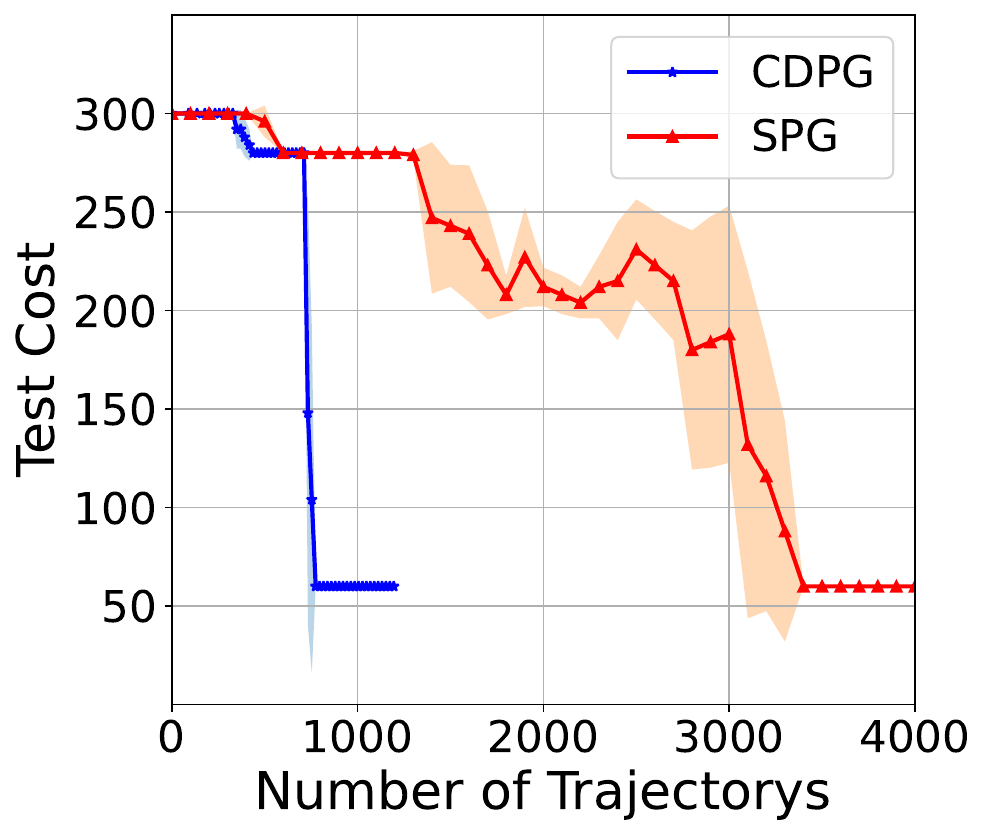}
}
\subfigure[Test cost v.s. \# iterations.]{
	\includegraphics[width=0.46\linewidth]{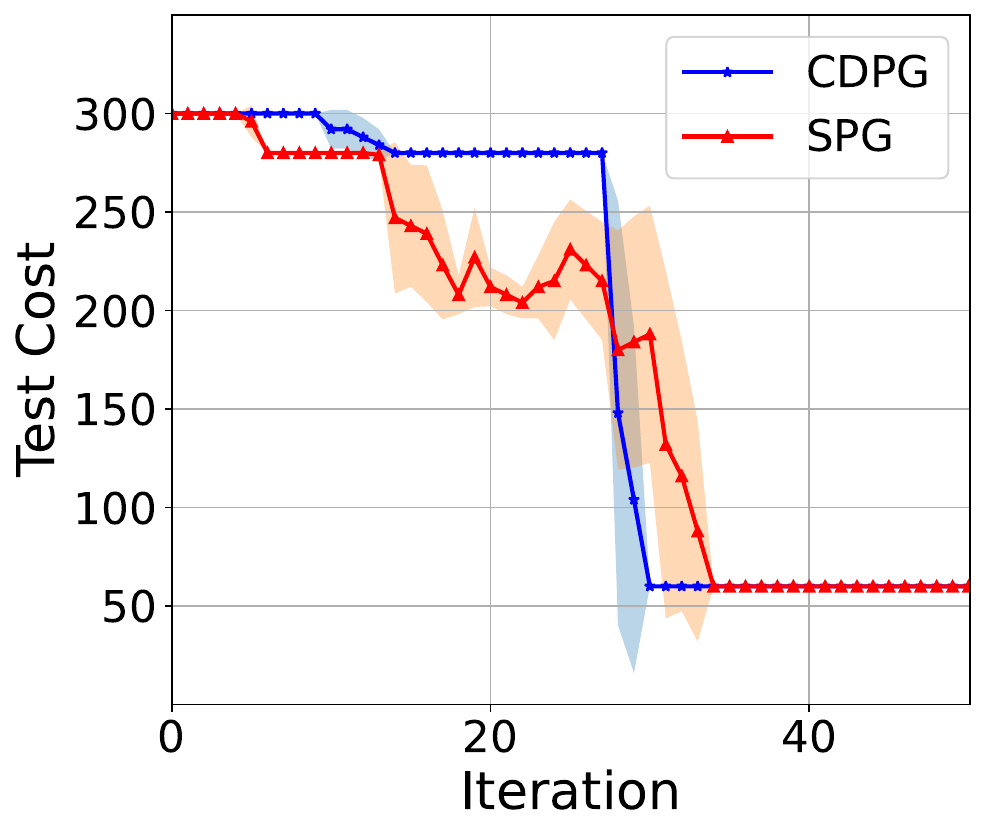}
}
\subfigure[Test cost v.s. time (sec.)]{
	\includegraphics[width=0.46\linewidth]{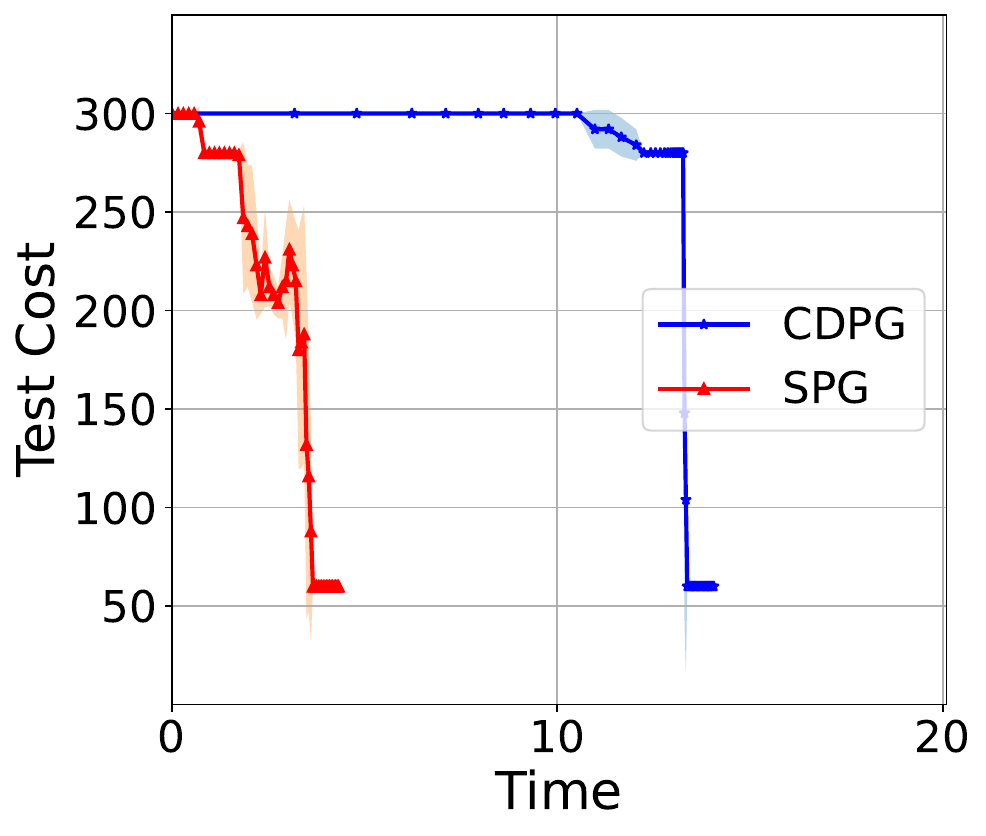}
}
\subfigure[Distance v.s. sample size.]{
	\includegraphics[width=0.46\linewidth]{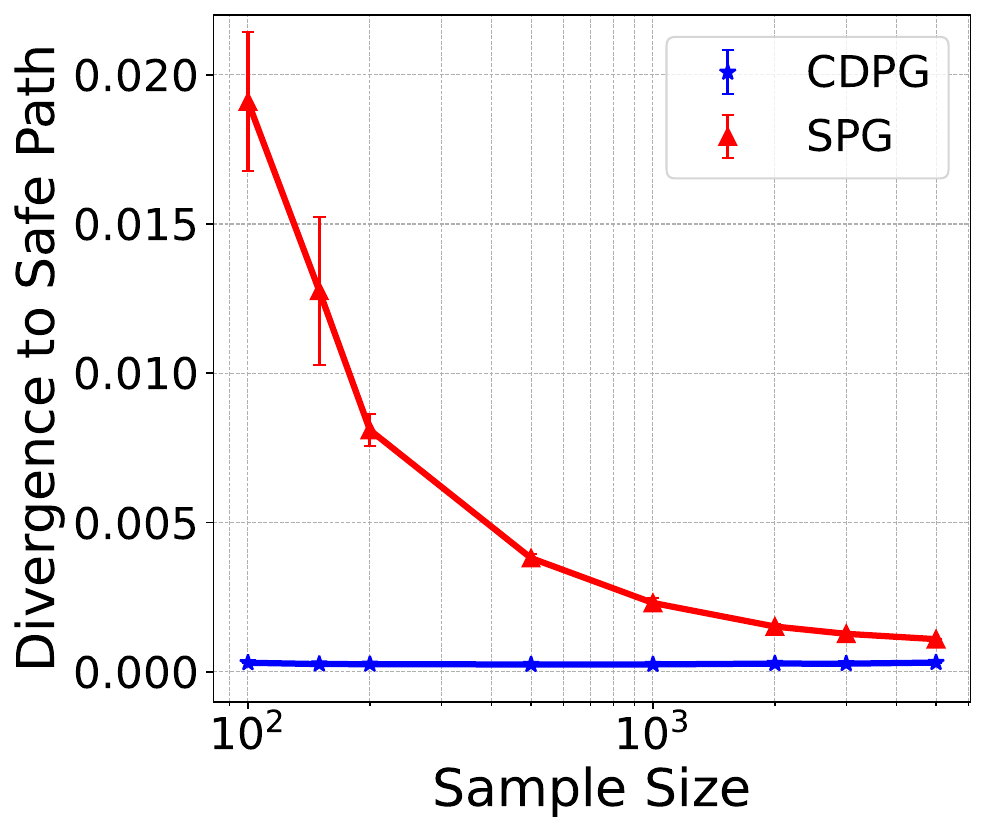}
}
\caption{CDPG vs. SPG~\cite{tamar2015policy} on Cliffwalking.}
    \label{figure:CliffWalking_CDPG_SPG}
\end{figure}

\section{Conclusions}
\label{section:Conclusion}
In this paper, we proposed a new distributional policy gradient framework for risk-sensitive MDPs with the static CVaR objective. By combining distributional policy evaluation and distributional policy improvement, we derived an analytical expression for the gradient of the probability measure of the cumulative cost, which enables the construction of policy gradient algorithms for a broad class of coherent risk measures. To facilitate practical implementation, we developed a categorical approximation that projects the return distribution onto a categorical family (CDPG). Furthermore, we established finite-iteration stationarity guarantees under inexact policy evaluation. Through experiments in the stochastic Cliffwalk environment, we demonstrated that the proposed CDPG outperforms the sample-based baseline SPG in terms of iteration and sample efficiency. There are several promising directions for future research. First, it is interesting to explore whether our framework could be extended to richer parametric distribution families, such as quantile-based distributional representations and Gaussian mixture models. Second, extending the proposed framework to continuous action spaces is another important direction for future investigation.

\addtolength{\textheight}{-12cm}   




\bibliography{Xian_bib,ref}
\bibliographystyle{IEEEtran}

\setlength{\textheight}{9.5in}
\onecolumn
\appendices


\section{Useful Properties of the Operators}\label{appendix:operators}
In this appendix, we present some useful properties of the pushforward and projection operators. We first provide the following properties of the pushforward operator $(b_{c, \gamma})_{\#}$:
\begin{proposition}\label{prop: pushforward props}
    The pushforward operator $(b_{c,\gamma})_{\#}$ has the following properties:
    \begin{itemize}
        \item $\nabla_{\theta}(b_{c,\gamma})_{\#}\eta_\theta=(b_{c,\gamma})_{\#}\nabla_{\theta}\eta_\theta$ for all $\eta_\theta\in\calM(\bbR)$;
        \item $(b_{c,\gamma})_{\#}(\sum_{s}p_s\eta_\theta)=\sum_{s}p_s(b_{c,\gamma})_{\#}\eta_\theta$ for all $\eta_\theta\in\calM(\bbR)$ and $p_s\in\bbR$.
    \end{itemize}
\end{proposition}
\begin{proof}
Given any set $A\subset\bbR$, by Definition \ref{def:pushforward}, we have
\begin{align*}
(b_{c, \gamma})_{\#}\nabla_\theta\eta_\theta(A) = \nabla_\theta\eta_\theta[(b_{c, \gamma})^{-1}(A)].
\end{align*}
Similarly, we have
\begin{align*}
\nabla_\theta(b_{c, \gamma})_{\#}\eta_\theta(A) = \nabla_\theta\left(\eta_\theta[(b_{c, \gamma})^{-1}(A)]\right).
\end{align*}
Hence, we have $\nabla_\theta(b_{c, \gamma})_{\#}\eta_\theta = (b_{c, \gamma})_{\#}\nabla_\theta\eta_\theta$. Also, we have
\begin{align*}
(b_{c, \gamma})_{\#}(\sum_{s}p_s\eta_\theta)(A) &= \bigg(\sum_{s}p_s\eta_\theta\bigg)[(b_{c, \gamma})^{-1}(A)] \\
&= \sum_{s}p_s\eta_\theta[(b_{c, \gamma})^{-1}(A)] = \sum_{s}p_s(b_{c, \gamma})_{\#}\eta_\theta(A),
\end{align*}
which completes the proof.
\end{proof}

We then provide the following properties of the projection operator $\Pi_\calC$:
\begin{proposition}\label{prop: projected props}
    The projected operator $\Pi_\calC$ has the following properties:
    \begin{itemize}
        \item $\nabla_{\theta}\Pi_\calC\eta_\theta = \Pi_\calC\nabla_{\theta}\eta_\theta$ for all $\eta_\theta\in\calM_N$;
        \item $\Pi_\calC(\sum_{s}p_s\eta_\theta)=\sum_{s}p_s\Pi_\calC\eta_\theta$ for all $\eta_\theta\in\calM_N$.
    \end{itemize}
\end{proposition}

\begin{proof}
Assume $\eta_\theta = \sum_{i=1}^{N}P_i^{\theta}\delta_{y_i}$. Since $\Pi_\calC(\sum_{i=1}^{N}P_i^{\theta}\delta_{y_i}) = \sum_{i=1}^{N}P_i^{\theta}\Pi_\calC(\delta_{y_i})$, we have
\begin{align*}
\Pi_\calC \nabla_\theta\eta_\theta =\Pi_\calC\bigg\{\nabla_\theta\bigg(\sum_{i=1}^{N}P^\theta_i\delta_{y_i}\bigg)\bigg\} = \Pi_\calC\bigg\{\sum_{i=1}^{N}\nabla_\theta P^\theta_i\delta_{y_i}\bigg\} = \sum_{i=1}^{N}\nabla_\theta P^\theta_i\Pi_\calC(\delta_{y_i})
\end{align*}
and
\begin{align*}
\nabla_\theta\Pi_\calC\eta_\theta = \nabla_\theta\Pi_\calC\bigg\{\sum_{i=1}^{N}P^\theta_i\delta_{y_i}\bigg\} = \nabla_\theta\bigg\{\sum_{i=1}^{N}P^\theta_i\Pi_\calC(\delta_{y_i})\bigg\} = \sum_{i=1}^{N}\nabla_\theta P^\theta_i\Pi_\calC(\delta_{y_i})
\end{align*}
Similarly, let $\eta_\theta = \sum_{i=1}^{N}P_i^{\theta}\delta_{y_i}$, then we have
\begin{align*}
\Pi_\calC(\sum_{s}p_s\eta_\theta)&=\Pi_\calC\bigg(\sum_s p_s\sum_{i=1}^{N}P_i^{\theta}\delta_{y_i}\bigg) = \sum_s\sum_{i=1}^{N}p_s P_i^{\theta}\Pi_\calC(\delta_{y_i}) \\
&= \sum_{s}p_s\sum_{i=1}^{N}P^{\theta}_i\Pi_\calC(\delta_{y_i}) = \sum_{s}p_s\Pi_\calC\eta_{\theta}
\end{align*}
\end{proof}

Combining Propositions \ref{prop: pushforward props} and \ref{prop: projected props}, we get the following properties of projected pushforward operator $\Pi_\calC(b_{c,\gamma})_{\#}$:
\begin{proposition}\label{prop: projected pushforward props}
    The projected pushforward operator $\Pi_\calC(b_{c,\gamma})_{\#}$ has the following properties:
    \begin{itemize}
        \item $\nabla_{\theta}\Pi_\calC(b_{c,\gamma})_{\#}\eta_{\theta}=\Pi_\calC(b_{c,\gamma})_{\#}\nabla_{\theta}\eta_{\theta}$ for all $\eta_\theta\in\calM_N$;
        \item $\Pi_\calC(b_{c,\gamma})_{\#}(\sum_{s}p_s\eta_{\theta})=\sum_{s}p_s\Pi_\calC(b_{c,\gamma})_{\#}\eta_{\theta}$ for all $\eta_{\theta}\in\calM_N$.
    \end{itemize}
\end{proposition}

\section{Omitted Proofs}
\label{appendix:proofs}
In this appendix, we present all the omitted proofs.
\subsection{Proofs in Section~\ref{section:Preliminaries}}
\label{appendix:proof-sec2}

\begin{replemma}{\ref{lemma:DistBellmanEq}}
For each state $s \in \calS$ and action $a\in\calA$, let $\eta_{\theta}^s$ and $\eta_{\theta}^{(s,a)}$ be the probability measures associated with the random variables $Z_{\theta}^s$ and $Z_{\theta}^{(s,a)}$. Then
    \begin{align*}
    \eta_{\theta}^{(s,a)} &= \sum_{s^{\prime}\in\calS}P(s'|s, a) \sum_{a'\in\calA}\pi_\theta(a'|s')(b_{C(s, a), \gamma})_{\#}\eta_{\theta}^{(s^\prime, a^\prime)}\\
    &=\sum_{s^{\prime}\in\calS}P(s'|s, a) (b_{C(s, a),\gamma})_{\#}\eta_{\theta}^{s^{\prime}}.
    \end{align*}
\end{replemma}
\begin{proof}
Given a deterministic cost function $C(s,a)$, we have
\begin{align*}
\eta_\theta^{(s, a)} &\overset{(i)}{=} (\calT^\pi\eta_\theta)^{(s, a)} \\
&\overset{(ii)}{=} \sum_{s' \in \calS}P(s'|s, a)\sum_{a' \in \calA}\pi_\theta(a'|s') (b_{C(s, a), \gamma})_{\#}\eta^{(s', a')}_\theta \\
&\overset{(iii)}{=} \sum_{s' \in \calS}P(s'|s, a)(b_{C(s, a), \gamma})_{\#}\eta^{s'}_\theta
\end{align*}
where $(i)$ is the distributional Bellman equation from \cite{rowland2018analysis}, $(ii)$ is based on the definition of the distributional Bellman operator, and $(iii)$ uses $\eta_{\theta}^s = \sum_{a\in\mathcal{A}}\pi_{\theta}(a|s)\eta_{\theta}^{(s,a)}$ and Proposition \ref{prop: pushforward props}.
\end{proof}

\subsection{Proofs in Section~\ref{section:DPG}}
\label{appendix:proof-sec3}

\begin{reptheorem}{\ref{theorem:DistPolicyGrad}}
Suppose Assumption \ref{assump:bounded-policy} holds. Let $\eta_\theta^{(s, a)} \in \calP(\bbR)$ denote the fixed point of $\calT^{\pi_{\theta}}$ in Proposition \ref{prop:DistBellmanOperatorContractionMapping} for any $s \in \calS$ and $a \in \calA$. Let $\tau_{\theta}$ be a trajectory that starts at $s_0$ under $\pi_{\theta}$. For any $t\ge 1$, let $\tau_{\theta}(s_0, s_t) := (s_0, a_0, c_0, \dots, s_{t-1}, a_{t-1}, c_{t-1}, s_t)$ be a t-step sub-trajectory of $\tau_{\theta}$ truncated at $s_t$. Then
\begin{align*}
\nabla_\theta \eta_\theta^{s} = \bbE_{\tau_{\theta}}\bigg[g(s_0) + \sum_{t=1}^{\infty}\calB^{\tau_{\theta}(s_0, s_t)}g(s_t)\bigg]
\end{align*}
where $g(s):= \sum_{a\in\calA}\nabla_\theta\pi_{\theta}(a|s) \eta_\theta^{(s, a)}$ and $\calB^{\tau_{\theta}(s_0, s_t)}$ is the $t$-step pushforward operator, defined as
$\calB^{\tau_{\theta}(s_0, s_t)} := (b_{{c_0}, \gamma})_{\#}\dots(b_{c_{t-1}, \gamma})_{\#} = (b_{c_{t-1}+\gamma c_{t-2} + \dots + \gamma^{t-1}c_0, \gamma^t})_{\#}$.
\end{reptheorem}
\begin{proof}
Denote $g(s) = \sum_{a}\nabla_\theta\pi(a|s)\cdot\eta_\theta^{(s, a)}$ for notation simplicity, then we have
\allowdisplaybreaks
\begin{align*}
\nabla_\theta \eta_\theta^{s_0} &\overset{(i)}{=} \nabla_\theta \bigg[\sum_{a_0} \pi(a_0|s_0)\cdot\eta_\pi^{(s_0,a_0)}\bigg] = \sum_{a_0} \bigg[\nabla_\theta \pi(a_0|s_0)\cdot \eta_\theta^{(s_0,a_0)} + \pi(a_0|s_0)\cdot\nabla_\theta \eta_\theta^{(s_0,a_0)}\bigg] \\
&\overset{(ii)}{=} \sum_{a_0} \bigg[\nabla_\theta\pi(a_0|s_0)\cdot\eta_\theta^{(s_0, a_0)} + \pi(a_0|s_0)\cdot\nabla_\theta\bigg(\sum_{s_1}P(s_1|s_0,a_0)(b_{C(s_0,a_0),\gamma})_{\#}\eta_\theta^{s_1}\bigg)\bigg] \\ 
&\overset{(iii)}{=} g(s_0) + \sum_{a_0}\pi(a_0|s_0)\sum_{s_1}P(s_1|s_0,a_0)(b_{C(s_0,a_0),\gamma})_{\#}\nabla_\theta \eta_\theta^{s_1} \\
&\overset{(iv)}{=} g(s_0) + \sum_{a_0}\pi(a_0|s_0)\sum_{s_1}P(s_1|s_0,a_0)(b_{C(s_0,a_0),\gamma})_{\#}g(s_1) \\
&\quad+\sum_{a_0}\pi(a_0|s_0)\sum_{s_1}P(s_1|s_0,a_0)\sum_{a_1}\pi(a_1|s_1)\sum_{s_2}P(s_2|s_1,a_1)\bigg[(b_{C(s_0,a_0),\gamma})_{\#}(b_{C(s_1,a_1),\gamma})_{\#}\bigg]g(s_2) \\
&\quad+ \dots\dots \\
&\overset{(v)}{=} \bbE_{\tau_{\theta}}\bigg[g(s_0) + \sum_{t=1}^{\infty}\calB^{\tau_{\theta}(s_0, s_t)}g(s_t)\bigg], 
\end{align*}
where $(i)$ follows because $\eta_\theta^{s_0}$ is a mixture of probabilities, $(ii)$ utilizes the distributional Bellman equation (Lemma~\ref{lemma:DistBellmanEq}), $(iii)$ holds because of Proposition \ref{prop: pushforward props}, and $(iv)$ results from an iterative expansion of $\nabla_\theta\eta_\theta^s$ with Proposition~\ref{prop: pushforward props} and $(v)$ holds because each trajectory $\tau_{\theta}=(s_0,a_0,c_0,s_1,a_1,c_1,\ldots,s_t)$ has a probability of $\pi(a_0|s_0)P(s_1|s_0,a_0)\pi(a_1|s_1)P(s_2|s_1,a_1)\cdots P(s_t|s_{t-1},a_{t-1})$. Furthermore, for any two pushforward operators and any measure $\nu\in\calM(\bbR)$, we have
\begin{align*}
   (b_{c_0,\gamma})_{\#}(b_{c_1,\gamma})_{\#}\nu(A)& = (b_{c_0,\gamma})_{\#}\nu(b_{c_1,\gamma}^{-1}(A))=\nu(b_{c_0,\gamma}^{-1}(b_{c_1,\gamma}^{-1}(A)))\\
    =&\nu((b_{c_1,\gamma}b_{c_0,\gamma})^{-1}(A))=(b_{c_1,\gamma}b_{c_0,\gamma})_\#\nu(A)=(b_{c_1+\gamma c_0,\gamma^2})_\#\nu(A), \ \forall A\subset\bbR
\end{align*}
Thus, $(b_{c_0,\gamma})_{\#}(b_{c_1,\gamma})_{\#}=(b_{c_1+\gamma c_0,\gamma^2})_\#$, and the multi-step pushforward operator can be combined as $\calB^{\tau_{\theta}(s_0, s_t)} = (b_{{c_0}, \gamma})_{\#}\dots(b_{c_{t-1}, \gamma})_{\#} = (b_{c_{t-1}+\gamma c_{t-2} + \dots + \gamma^{t-1}c_0, \gamma^t})_{\#}$.
\end{proof}

\begin{repcorollary}{\ref{corollary:MeasureToPDF}}
Suppose $\nabla_\theta \eta_\theta^s$ is well-defined, and both 
$\frac{\partial}{\partial x}\frac{\partial}{\partial \theta}F_{Z_{\theta}^s}(x)$
and $\frac{\partial}{\partial \theta} f_{Z_{\theta}^s}(x)$ 
are continuous. Then,
\begin{align*}
   \frac{\partial}{\partial \theta} f_{Z_{\theta}^s}(x) = 
   \frac{\partial}{\partial x}\nabla_\theta \eta_\theta^s\bigl((-\infty, x]\bigr).
\end{align*}
\end{repcorollary}
\begin{proof}
From Theorem \ref{theorem:DistPolicyGrad}, $g(s)$ is a vector of signed measures. Since pushforward and expectation operators preserve signed measures and the series in Theorem 3.1 is well defined, $\nabla_{\theta}\eta_{\theta}^s$ is a vector of signed measures.
Furthermore, since $\eta_\theta^s(\Omega)=1$, we have $\nabla_\theta \eta_\theta^s(\Omega)=0$, i.e., $\nabla_\theta \eta_\theta^s$ has a total mass of 0.
From the definition of probability measure $\eta_{\theta}^s$, we have $\eta_{\theta}^s((-\infty, x])=\bbP\{\omega\in\Omega: Z_{\theta}^s(\omega)\in (-\infty,x]\}=F_{Z_\theta^s}(x)$. Taking derivative with respect to $\theta$ on both sides, we have
\begin{align}
    \nabla_\theta\eta_{\theta}^s((-\infty, x])=\frac{\partial}{\partial \theta}F_{Z_{\theta}}(x)
    \label{eq:partial-theta}
\end{align}
Now taking the derivative with respect to $x$ again, we have
\begin{align*}
    \frac{\partial}{\partial x}\nabla_\theta\eta_{\theta}^s((-\infty, x])=\frac{\partial}{\partial x}\frac{\partial}{\partial \theta}F_{Z_{\theta}^s}(x)
\end{align*}
Since $\frac{\partial}{\partial x}\frac{\partial}{\partial \theta}F_{Z_{\theta}^s}(x)$ and $\frac{\partial}{\partial \theta}f_{Z_{\theta}^s}(x)$ are continuous, we can switch the order of partial derivatives and get
\begin{align*}
\frac{\partial}{\partial x}\nabla_\theta\eta_{\theta}^s((-\infty, x])=\frac{\partial}{\partial x}\frac{\partial}{\partial \theta}F_{Z_{\theta}^s}(x)=\frac{\partial}{\partial \theta}\frac{\partial}{\partial x}F_{Z_{\theta}^s}(x)=\frac{\partial}{\partial \theta}f_{Z_{\theta}^s}(x).
\end{align*}
This completes the proof.
\end{proof}


\subsection{Proofs in Section~\ref{section:CDPG}}
\label{appendix:proof-sec4}

\begin{replemma}{\ref{lemma:ProjDistBellmanEquation}}
Let $\eta_{N, \infty} \in \calP_N^{\calS \times \calA}$ be the fixed point of $\Pi_{\calC}\calT^{\pi}$. Then, for any $s \in \calS$ and $a \in \calA$,
\begin{align*}
    \eta_{N, \infty}^{(s,a)} = \sum_{s'} P(s'|s, a)\Pi_\calC \bigl(b_{C(s,a),\gamma}\bigr)_\# \eta_{N, \infty}^{s'}.
\end{align*}
\end{replemma}
\begin{proof}
We have
    \begin{align*}
    \eta_{N, \infty}^{(s,a)}&\overset{(i)}{=} \Pi_\calC(\sum_{s^{\prime}\in\calS}P(s'|s, a) \sum_{a'\in\calA}\pi_\theta(a'|s')(b_{C(s, a), \gamma})_{\#}\eta_{N, \infty}^{(s^\prime, a^\prime)})\\
    &\overset{(ii)}{=} \sum_{s'\in\calS}P(s'|s, a)\Pi_\calC(\sum_{a'\in\calA}\pi_\theta(a'|s')(b_{C(s, a), \gamma})_{\#}\eta_{N, \infty}^{(s' ,a')}) \\
    &\overset{(iii)}{=} \sum_{s'\in\calS}P(s'|s, a)\Pi_\calC((b_{C(s, a), \gamma})_{\#}\eta_{N, \infty}^{s'})
    \end{align*}
    where $(i)$ is because $\eta_{N, \infty}$ is the fixed point of $\Pi_\calC\calT^{\pi}$; $(ii)$ holds due to Proposition \ref{prop: projected props}; and $(iii)$ follows from Proposition~\ref{prop: pushforward props} and $\sum_{a' \in \calA}\pi_\theta(a'|s')\eta_{N,\infty}^{(s', a')} = \eta_{N,\infty}^{s'}$.
\end{proof}

Before proving Lemma \ref{lemma:OptimGap}, we first state the following supporting lemmas.
\begin{lemma}[Proposition 3, \cite{rowland2018analysis}]
\label{lemma:CramerDistanceLemma}
Let $\eta$ and $\eta_{N, \infty}$ be the limiting return distribution of $\calT^\pi$ and $\Pi_{\calC}\calT^\pi$, respectively. If $\eta^{(s, a)}$ is supported on $[z_1, z_N]$ for all $(s, a) \in \calS \times \calA$, then we have
\begin{align*}
l_2^2(\eta_{N,\infty}^{(s,a)}, \eta^{(s,a)}) \leq \frac{1}{1-\gamma}\frac{z_N-z_1}{N-1},\ \forall (s, a) \in \calS \times \calA
\end{align*}
\end{lemma}

\begin{lemma}[Cauchy Schwarz Inequality]
\label{lemma:CauchySchwarz}
For any square-integrable functions $f$ and $g$, we have
\begin{align*}
\bigg|\int_{a}^{b}f(x)g(x)dx\bigg|^2 \leq \bigg(\int_{a}^{b}|f(x)|^2dx\bigg)\bigg(\int_{a}^{b}|g(x)|^2dx\bigg)
\end{align*}
\end{lemma}

\begin{lemma}
\label{lemma:MixtureCramerDist}
If the Cramér distance $l_2(\eta_1^{(s,a)}, \eta_2^{(s,a)}) \leq \epsilon$ for all $s \in \calS$ and $a \in \calA$, then the mixture distributions $\eta_1^s=\sum_{a \in \calA}\pi(a|s)\cdot\eta_1^{(s, a)}$ and $\eta_2^s=\sum_{a \in \calA}\pi(a|s)\cdot\eta_2^{(s, a)}$ satisfy:
\begin{align*}
l_2\left(\eta_1^s, \eta_2^s\right) \leq \epsilon
\end{align*}
\end{lemma}
\begin{proof}
Let $F_1^{(s,a)}(x)$ and $F_2^{(s,a)}(x)$ denote the CDFs of $\eta_1^{(s,a)}$ and $\eta_2^{(s,a)}$ respectively. Then we have
\begin{align*}
l_2(\eta_1^{(s,a)}, \eta_2^{(s,a)}) = \sqrt{\int_{\mathbb{R}} \left[ F_1^{(s,a)}(x) - F_2^{(s,a)}(x) \right]^2 dx} \leq \epsilon
\end{align*}
The mixture distributions' CDFs are:
\begin{align*}
F_1^s(x) &= \sum_{a \in \mathcal{A}} \pi(a|s) F_1^{(s,a)}(x), \\
F_2^s(x) &= \sum_{a \in \mathcal{A}} \pi(a|s) F_2^{(s,a)}(x).
\end{align*}
Their squared Cramér distance becomes:
\begin{align*}
l_2^2\left(\eta_1^s, \eta_2^s\right) &= \int_{\mathbb{R}} \left[ \sum_{a} \pi(a|s)\left(F_1^{(s,a)}(x) - F_2^{(s,a)}(x)\right) \right]^2 dx.
\end{align*}
By Cauchy-Schwarz inequality, we have
\begin{align*}
\left[ \sum_{a} \pi(a|s)\left(F_1^{(s,a)}(x) - F_2^{(s,a)}(x)\right) \right]^2 &\leq \left( \sum_{a} \pi(a|s) \right) \left( \sum_{a} \pi(a|s)\left[ F_1^{(s,a)}(x) - F_2^{(s,a)}(x) \right]^2 \right) \\
&= \sum_{a} \pi(a|s)\left[ F_1^{(s,a)}(x) - F_2^{(s,a)}(x) \right]^2.
\end{align*}
Hence, we have
\begin{align*}
l^2_2\left(\eta_1^s, \eta_2^s\right)&\leq \sum_{a} \pi(a|s) \int_{\mathbb{R}} \left[ F_1^{(s,a)}(x) - F_2^{(s,a)}(x) \right]^2 dx \\
&= \sum_{a} \pi(a|s) \cdot l^2_2(\eta_1^{(s,a)}, \eta_2^{(s,a)}) \\
&\leq \sum_{a} \pi(a|s) \cdot \epsilon^2 = \epsilon^2,
\end{align*}
which completes the proof.
\end{proof}

\begin{replemma}{\ref{lemma:OptimGap}}
For any $\epsilon_{opt} > 0$, we have
$|\min\limits_{\theta}\rho_{\text{CVaR}}(Z^s) - \min\limits_{\theta}\rho_{\text{CVaR}}(Z_{N}^s)| \leq \epsilon_{opt}$,
whenever 
\begin{align*}
N \geq \frac{(z_{\max}-z_{\min})^2}{\alpha^2(1-\gamma) \epsilon_{opt}^2}+1.
\end{align*}
\end{replemma}
\begin{proof}
Let $\eta^{s}$ and $\eta_{N, \infty}^s$ be the limiting distribution of $Z^s$ and $Z_{N}^s$, respectively. By Lemmas~\ref{lemma:CramerDistanceLemma} and \ref{lemma:MixtureCramerDist}, we have
\begin{align*}
\bar{l}_2^2(\eta^s, \eta_{N, \infty}^s) \leq \frac{1}{1-\gamma}\frac{z_N-z_1}{N-1}
\end{align*}
where $l_2$ is the Cramer distance, defined as
\begin{align*}
l_2^2(\eta_{N, \infty}^s, \eta^s) = \int_{z_1}^{z_N}[F_{N, \infty}^s(x) - F^s(x)]^2dx
\end{align*}
By Lemma~\ref{lemma:CauchySchwarz} (Cauchy Schwarz Inequality), we have
\begin{align*}
\|F_{N, \infty}^s - F^s\|_1^2=\bigg(\int_{z_1}^{z_N}|F_{N, \infty}^s(x) - F^s(x)|dx\bigg)^2\le (z_N - z_1)\int_{z_1}^{z_N}|F_{N, \infty}^s(x) - F^s(x)|^2dx\le \frac{(z_N-z_1)^2}{(1-\gamma)(N-1)}
\end{align*}
According to \cite{liang2024regret}, CVaR is $L_1$-Lipschitz continuous with $L_1=\frac{1}{\alpha}$, i.e.,
\begin{align*}
[\rho_{\text{CVaR}}(Z_N^s) - \rho_{\text{CVaR}}(Z^s)]^2 \leq \frac{1}{\alpha^2}\|F_{N, \infty}^s - F^s\|_1^2
\end{align*}
Hence, we have
\begin{align*}
[\rho_{\text{CVaR}}(Z_{\theta, N}^s) - \rho_{\text{CVaR}}(Z_{\theta}^s)]^2 \leq \frac{1}{1-\gamma}\frac{(z_N-z_1)^2}{\alpha^2(N-1)} 
\end{align*}
If we set $N \ge \frac{1}{1-\gamma}\frac{(z_N-z_1)^2}{\alpha^2\epsilon_{\text{opt}}^2}+1 = \mathcal{O}(\epsilon_{\text{opt}}^{-2})$, then $|\rho_{\text{CVaR}}(Z_{N}^s) - \rho_{\text{CVaR}}(Z^s)|\le \epsilon_{\text{opt}}$ for all $\theta\in\Theta$. Denote $\theta^*=\arg\min_{\theta}\rho_{\text{CVaR}}(Z^s)$ and $\theta_N^*=\arg\min_{\theta}\rho_{\text{CVaR}}(Z_N^s)$. From the optimality of $\theta^*$ and $\theta^*_N$, we have
\begin{align*}
|\min_{\theta}\rho_{\text{CVaR}}(Z^s) - \min_{\theta}\rho_{\text{CVaR}}(Z_N^s)| & = |\rho_{\text{CVaR}}(Z_{\theta^*}^s)-\rho_{\text{CVaR}}(Z_{\theta^*_N, N}^s)|\\
&\leq \max\big\{\rho_{\text{CVaR}}(Z_{\theta^*}^s) - \rho_{\text{CVaR}}(Z_{\theta^*_N, N}^s), \rho_{\text{CVaR}}(Z_{\theta^*_N, N}^s) - \rho_{\text{CVaR}}(Z_{\theta^*}^s)\big\} \\
&\leq \max\big\{\rho_{\text{CVaR}}(Z_{\theta_N^\ast}^s) - \rho_{\text{CVaR}}(Z_{{\theta_N^\ast}, N}^s), \rho_{\text{CVaR}}(Z_{\theta^*, N}^s) - \rho_{\text{CVaR}}(Z_{\theta^*}^s)\big\} \\
&\leq \epsilon_{\text{opt}},
\end{align*}
which completes the proof.
\end{proof}

\begin{reptheorem}{\ref{theorem:CatPolicyGrad}}
Let $\eta_{N, \infty}^{(s, a)} \in \calP_N$ denote the fixed point of $\Pi_{\calC}\calT^{\pi}$ for any $s \in \calS$ and $a \in \calA$. Consider a trajectory $\tau_{\theta}$ starting from $s_0$ under policy $\pi_\theta$. For any $t\ge 1$, let $\tau_{\theta}(s_0,s_t)$ be the $t$-step sub-trajectory truncated at $s_t$. Then
\begin{align*}
\nabla_\theta \eta_{N, \infty}^{s_0} 
= \mathbb{E}_{\tau_{\theta}}\bigg[g_{N,\infty}(s_0) + \sum_{t=1}^{\infty}\tilde{\mathcal{B}}^{\tau_{\theta}(s_0,s_t)}g_{N,\infty}(s_t)\bigg],
\end{align*}
where $g_{N,\infty}(s):=\sum_{a \in \calA}\nabla_\theta\pi_\theta(a|s)\eta_{N, \infty}^{(s,a)}$, and $\tilde{\mathcal{B}}^{\tau_{\theta}(s_0,s_t)}$ is the $t$-step projected pushforward operator defined by $\tilde{\calB}^{\tau_{\theta}(s_0, s_t)} = \Pi_\calC(b_{c_{0}, \gamma})_{\#}\Pi_\calC(b_{c_{1}, \gamma})_{\#}\dots\Pi_\calC(b_{c_{t-1}, \gamma})_{\#}$.
\end{reptheorem}
\begin{proof}
Denote $g_{N,\infty}(s) = \sum_{a}\nabla_\theta\pi_\theta(a|s)\cdot\eta_{N,\infty}^{(s, a)}$ for notation simplicity, then we have
\begin{align*}
&\nabla_\theta \eta_{N,\infty}^{s_0} = \nabla_\theta \bigg[\sum_{a_0} \pi_\theta(a_0|s_0)\cdot\eta_{N,\infty}^{(s_0,a_0)}\bigg] = \sum_{a_0} \bigg[\nabla_\theta \pi_\theta(a_0|s_0)\cdot\eta_{N,\infty}^{(s_0,a_0)} + \pi_\theta(a_0|s_0)\cdot\nabla_\theta\eta_{N,\infty}^{(s_0,a_0)}\bigg] \\
&\overset{(i)}{=} \sum_{a_0} \bigg[\nabla_\theta\pi_\theta(a_0|s_0)\cdot\eta_{N,\infty}^{(s_0, a_0)} + \pi_\theta(a_0|s_0)\cdot\nabla_\theta\bigg(\sum_{s_1}P(s_1|s_0, a_0)\Pi_\calC(b_{C(s_0, a_0),\gamma})_{\#}\eta_{N,\infty}^{s_1}\bigg)\bigg] \\ 
&\overset{(ii)}{=} g_{N,\infty}(s_0) + \sum_{a_0}\pi_\theta(a_0|s_0)\sum_{s_1}P(s_1|s_0, a_0)\Pi_\calC(b_{C(s_0, a_0),\gamma})_{\#}\nabla_\theta\eta_{N,\infty}^{s_1} \\
&\overset{(iii)}{=} g_{N,\infty}(s_0) + \sum_{a_0}\pi_\theta(a_0|s_0)\sum_{s_1}P(s_1|s_0, a_0)\Pi_\calC(b_{C(s_0, a_0),\gamma})_{\#}g_{N,\infty}(s_1) \\
&\quad+\sum_{a_0}\pi_\theta(a_0|s_0)\sum_{s_1}P(s_1|s_0, a_0)\sum_{a_1}\pi_\theta(a_1|s_1)\sum_{s_2}P(s_2|s_1, a_1)\Pi_\calC(b_{C(s_0, a_0),\gamma})_{\#}\Pi_\calC(b_{C(s_1, a_1),\gamma})_{\#}g_{N,\infty}(s_2) \\
&\quad+ \dots\dots \\
&\overset{(iv)}{=} \mathbb{E}_{\tau_{\theta}}\bigg[g_{N,\infty}(s_0) + \sum_{t=1}^{\infty}\tilde{\mathcal{B}}^{\tau_{\theta}(s_0,s_t)}g_{N,\infty}(s_t)\bigg],
\end{align*}
where $(i)$ is due to the projected distributional Bellman equation (Lemma \ref{lemma:ProjDistBellmanEquation}); $(ii)$ is due to Proposition \ref{prop: projected pushforward props}; $(iii)$ results from an iterative expansion of $\nabla_\theta\eta_{N,\infty}^{s_1}$ with Proposition~\ref{prop: projected pushforward props} and $(iv)$ holds because each trajectory $\tau_{\theta}=(s_0,a_0,c_0,s_1,a_1,c_1,\ldots,s_t)$ has a probability of $\pi(a_0|s_0)P(s_1|s_0,a_0)\pi(a_1|s_1)P(s_2|s_1,a_1)\cdots P(s_t|s_{t-1},a_{t-1})$.
\end{proof}

\begin{replemma}{\ref{lemma:CVaRSmooth}}
Under Assumptions~\ref{assumption:bounded-derivative}--\ref{assumption:AlphaQuantile}, the objective function~\eqref{eq:CatApproxProblem} is $\beta$-smooth.
\end{replemma}
\begin{proof}
From Eq.~\eqref{eq:CVaR_Grad}, the gradient of CVaR is written as follows
\begin{align*}
    \nabla_{\theta}\rho_{\text{CVaR}}(Z_{N}^s) = \frac{1}{\alpha}
    \sum_{i=1}^{N}\nabla_{\theta}p^{\theta}_i\bigl(z_i - q_\alpha^{\theta}\bigr) \mathbf{1}_{\{z_i > q_\alpha^{\theta}\}}.
\end{align*}
Under Assumption~\ref{assumption:AlphaQuantile}, when we perturb $\theta$ slightly, the $(1-\alpha)$-quantile $q_\alpha^{\theta}$ will not change, which means that we have a fixed set $\{i:z_i>q_\alpha^{\theta}\}$. Denote this set as $\calT_{1-\alpha}$. Then we have
\begin{align*}
    \nabla_{\theta}^2\rho_{\text{CVaR}}(Z_{N}^s) = \frac{1}{\alpha}
    \sum_{i\in\calT_{1-\alpha}}\nabla_{\theta}^2p^{\theta}_i\bigl(z_i - q_\alpha^{\theta}\bigr).
\end{align*}
By Assumption~\ref{assumption:bounded-derivative}, we can bound $\nabla_{\theta}^2\rho_{\text{CVaR}}(Z_{N}^s)$ as follows
\begin{align*}
    ||\nabla_{\theta}^2\rho_{\text{CVaR}}(Z_{N}^s)||_2 \le \frac{1}{\alpha}
    |\calT_{1-\alpha}|||\nabla_{\theta}^2p^{\theta}_i||_2\max_{i=1,\ldots,N}\bigl|z_i - q_\alpha^{\theta}\bigr|\le \frac{N}{\alpha}C_P^{(2)}|z_{max}-z_{min}|.
\end{align*}
Setting $\beta:=\frac{N}{\alpha}C_P^{(2)}|z_{max}-z_{min}|$, we obtain $||\nabla_{\theta}\rho_{\text{CVaR}}(Z_{N}^s)|_{\theta=\theta_1}-\nabla_{\theta}\rho_{\text{CVaR}}(Z_{N}^s)|_{\theta=\theta_2}||\le\beta||\theta_1-\theta_2||$, which completes the proof.
\end{proof}

We first state a supporting lemma for the gradient descent algorithm with inexact gradients.

\begin{lemma}
\label{lemma:ConvergenceWithError}
Let $f: \mathbb{R}^d \to \mathbb{R}$ be a $\beta$-smooth function with a lower bound $f^* = \inf_{x} f(x)$. Consider the gradient descent update with errors:
\begin{align*}
x_{t+1} = x_t - \eta (\nabla f(x_t) + \epsilon_t),
\end{align*}
where $\epsilon_t$ is the gradient error at iteration $t$. Suppose the step size is chosen as $\eta = \frac{1}{\beta}$, and the errors satisfy $\|\epsilon_t\|_2 < C$ for all iterations $t$. Then:
\begin{align*}
\frac{1}{T}\sum_{t=1}^{T}\|\nabla f(x_t)\|_2^2 \leq \frac{2\beta(f(x_1)-f^*)}{T} + C^2
\end{align*}
\end{lemma}

\begin{proof}
Starting from the $\beta$-smoothness condition (see Eq.~(2.4) in \cite{bertsekas2000gradient}):
\begin{align*}
f(x_{t+1}) \leq f(x_t) + \nabla f(x_t)^\top (x_{t+1} - x_t) + \frac{\beta}{2} \|x_{t+1} - x_t\|_2^2.
\end{align*}
Substitute the update rule $x_{t+1} - x_t = -\eta (\nabla f(x_t) + \epsilon_t)$:
\begin{align*}
f(x_{t+1}) \leq f(x_t) - \eta \nabla f(x_t)^\top (\nabla f(x_t) + \epsilon_t) + \frac{\beta \eta^2}{2} \|\nabla f(x_t) + \epsilon_t\|_2^2.
\end{align*}
Expand the terms and set $\eta = \frac{1}{\beta}$:
\begin{align*}
f(x_{t+1}) \leq f(x_t) - \frac{1}{\beta} \|\nabla f(x_t)\|_2^2 - \frac{1}{\beta} \nabla f(x_t)^\top \epsilon_t + \frac{1}{2\beta} \left( \|\nabla f(x_t)\|_2^2 + 2 \nabla f(x_t)^\top \epsilon_t + \|\epsilon_t\|_2^2 \right).
\end{align*}
Simplify the inequality by canceling cross terms:
\begin{align*}
f(x_{t+1}) \leq f(x_t) - \frac{1}{2\beta} \|\nabla f(x_t)\|_2^2 + \frac{1}{2\beta} \|\epsilon_t\|_2^2.
\end{align*}
And since $\|\epsilon_t\|_2 \leq C$, we have
\begin{align*}
f(x_{t+1}) \leq f(x_t) - \frac{1}{2\beta} \|\nabla f(x_t)\|_2^2 + \frac{C^2}{2\beta}. 
\end{align*}
Summing over $t = 1$ to $T$:
\begin{align*}
f(x_{T+1}) - f(x_1) \leq -\frac{1}{2\beta} \sum_{t=1}^T \|\nabla f(x_t)\|_2^2 + \frac{TC^2}{2\beta}
\end{align*}
Hence, the average gradient norm is bounded by
\begin{align*}
\frac{1}{T}\sum_{t=1}^{T}\|\nabla f(x_t)\|_2^2 \leq \frac{2\beta(f(x_1)-f^*)}{T} + C^2
\end{align*}
\end{proof}

Next, we quantify how the approximation error incurred by replacing $\eta_{N,\infty}$ with $\eta_{N,k}$ propagates to the approximations of $\nabla_{\theta}\eta_{N,\infty}$ and $\nabla_{\theta}\rho_{\text{CVaR}}(Z_{N,\infty}^s)$. The corresponding error bounds are established in the following three lemmas.
\begin{lemma}\label{lem:projection-error}
Let $\eta_{N,\infty}=\sum_{i=1}^Np_i^{N,\infty}\delta_{z_i}$ be the limiting  distribution induced by the operator $\Pi_{\mathcal{C}} \mathcal{T}^\pi$ on the finite support $\{z_1,\dots,z_N\}$. For any initial distribution $\eta_{N,0}$, define $\eta_{N,k}:=(\Pi_{\mathcal{C}} \mathcal{T}^\pi)^k\eta_{N,0}$. Let $\beta = \max\{c_{\max}-z_{\min}, z_{\max}-c_{\min}\}$, and $\mu = \max\{|z_{\max}|, |z_{\min}|\}$. Denote $\big\|\eta_{N, \infty} - \eta_{N, k}\big\|_\infty := \max_{j \in [N]}\big|p_j^{N, \infty} - p_j^{N, k}\big|$. Then, for the one-step projected pushforward operator $\Pi_{\calC}(b_{c,\gamma})_{\#}$, we have
\begin{align*}
\bigl\|\Pi_{\mathcal{C}}\bigl(b_{c,\gamma}\bigr)_{\#}\eta_{N,\infty}
-\Pi_{\mathcal{C}}\bigl(b_{c,\gamma}\bigr)_{\#}\eta_{N,k}\bigr\|_{\infty}
\leq \delta_{\Pi}\bigl\|\eta_{N,\infty}-\eta_{N,k}\bigr\|_{\infty}\le 2\delta_\Pi C(N,k),
\end{align*}
where $\delta_{\Pi}:=\frac{2(\gamma+1)(\beta + \gamma\mu)(N-1)}{\gamma(z_N-z_1)}$ is an \emph{error amplification coefficient} arising from the projection $\Pi_{\mathcal{C}}$ and $C(N,k):=\sqrt{\frac{N}{z_N-z_1}\gamma^{k}\delta_0}$ with $\delta_0:=l_2^2(\eta_{N,0},\eta_{N,\infty})$. Furthermore,
\begin{align*}
\bigl\|\Pi_{\mathcal{C}}\bigl(b_{c_1,\gamma}\bigr)_{\#}\ldots\Pi_{\mathcal{C}}\bigl(b_{c_h,\gamma}\bigr)_{\#}\eta_{N,\infty}
-\Pi_{\mathcal{C}}\bigl(b_{c_1,\gamma}\bigr)_{\#}\ldots\Pi_{\mathcal{C}}\bigl(b_{c_h,\gamma}\bigr)_{\#}\eta_{N,k}\bigr\|_{\infty}
\leq (\delta_{\Pi})^h\bigl\|\eta_{N,\infty}-\eta_{N,k}\bigr\|_{\infty}\le 2\delta_\Pi^h C(N,k).
\end{align*}
\end{lemma}
\begin{proof}
Denote $\delta_0:=l_2^2(\eta_{N,0},\eta_{N,\infty})$. By Proposition~\ref{prop:ProjDistBellmanOperatorContractionMapping}, we have
\begin{align*}
l^2_2(\eta_{N,k},\eta_N)=\frac{z_N-z_1}{N-1}\bigg(\big|p_1^{N, k} - p_1^{N, \infty}\big|^2 + \big|\sum_{i=1}^{2}(p_i^{N, k} - p_i^{N, \infty})\big|^2 + \dots + \big|\sum_{i=1}^{N}(p_i^{N, k} - p_i^{N, \infty})\big|^2) \leq \gamma^{k}\delta_0.
\end{align*}
Consequently, we have 
\begin{align}
\big|\sum_{i=1}^{j}(p_i^{N, k} - p_i^{N, \infty})\big|\le \underbrace{\sqrt{\frac{N}{z_N-z_1}\gamma^{k}\delta_0}}_{C(N, k)},\  \forall j=1,\ldots, N.\label{eq:CDF}
\end{align}
As a result,
\begin{align*}
\big|p_j^{N, \infty} - p_j^{N, k}\big| =  \left|\sum_{i=1}^{j}(p_i^{N, k} - p_i^{N, \infty}) - \sum_{i=1}^{j-1}(p_i^{N, k} - p_i^{N, \infty})\right|\leq 2\underbrace{\sqrt{\frac{N}{z_N-z_1}\gamma^{k}\delta_0}}_{C(N, k)}, \quad\forall j \in [N].
\end{align*}
Denote $\big\|\eta_{N, \infty} - \eta_{N, k}\big\|_\infty =\big\|p^{N, \infty} - p^{N, k}\big\|_\infty = \max_{j \in [N]}\big|p_j^{N, \infty} - p_j^{N, k}\big|$. Let $\Pi_{\calC}(b_{c, \gamma})_{\#}\eta_{N, k}$ and $\Pi_{\calC}(b_{c, \gamma})_{\#}\eta_{N, \infty}$ be the probability distributions after applying one step of projected pushforward operator to $\eta_{N, k}$ and $\eta_{N, \infty}$, respectively. Consider any specific support point $z_i$. Define $\calL_i = \{j\in[N]: c + \gamma z_j \in [z_{i-1}, z_i)\}$ and $\calR_i = \{j\in[N]: c + \gamma z_j \in [z_{i}, z_{i+1})\}$. The cardinality of $\calL_i$ and $\calR_i$ can be bounded as follows:
\begin{align*}
\calL_i:\ c + \gamma z_j \geq z_{i-1} \text{ and } c + \gamma z_j < z_{i} \Longrightarrow z_j \geq \frac{z_{i-1}-c}{\gamma}, z_j <\frac{z_i-c}{\gamma} \Longrightarrow |\calL_i| \leq \frac{1}{\gamma} + 1 \\
\calR_i:\ c + \gamma z_j \geq z_{i} \text{ and } c + \gamma z_j < z_{i+1} \Longrightarrow z_j \geq \frac{z_{i}-c}{\gamma}, z_j < \frac{z_{i+1}-c}{\gamma} \Longrightarrow |\calR_i| \leq \frac{1}{\gamma} + 1
\end{align*}
Let $|z| = \frac{z_N-z_1}{N-1}=z_i-z_{i-1},\ \forall i$. According to the definition of the projection and pushforward operator, the probability mass of $\Pi_{\calC}(b_{c, \gamma})_{\#}\eta_{N, k}$ and $\Pi_{\calC}(b_{c, \gamma})_{\#}\eta_{N, \infty}$ at the support $z_i$ can be computed by
\begin{align*}
\Pi_{\calC}(b_{c, \gamma})_{\#}\eta_{N, \infty}(z_i) &= \sum_{j \in \calL_i}\frac{(c+\gamma z_j) - z_{i-1}}{z_i-z_{i-1}}p_j^{N, \infty} + \sum_{j \in \calR_i}\frac{z_{i+1}-(c+\gamma z_j)}{z_{i+1}-z_i}p_j^{N, \infty} \\
&=\frac{c-z_{i-1}}{|z|}\sum_{j \in \calL_i}p_j^{N, \infty} + \frac{z_{i+1}-c}{|z|}\sum_{j \in \calR_i}p_j^{N, \infty} + \gamma\sum_{j \in \calL_i}\frac{z_j}{|z|}p_j^{N, \infty} - \gamma\sum_{j \in \calR_i}\frac{z_j}{|z|}p_j^{N, \infty}
\end{align*}
and
\begin{align*}
\Pi_{\calC}(b_{c, \gamma})_{\#}\eta_{N, k}(z_i) &= \sum_{j \in \calL_i}\frac{(c+\gamma z_j) - z_{i-1}}{z_i-z_{i-1}}p_j^{N, k} + \sum_{j \in \calR_i}\frac{z_{i+1}-(c+\gamma z_j)}{z_{i+1}-z_i}p_j^{N, k} \\
&=\frac{c-z_{i-1}}{|z|}\sum_{j \in \calL_i}p_j^{N, k} + \frac{z_{i+1}-c}{|z|}\sum_{j \in \calR_i}p_j^{N, k} + \gamma\sum_{j \in \calL_i}\frac{z_j}{|z|}p_j^{N, k} - \gamma\sum_{j \in \calR_i}\frac{z_j}{|z|}p_j^{N, k}
\end{align*}
Let $\beta = \max\{c_{\max}-z_{\min}, z_{\max}-c_{\min}\}$, and $\mu = \max\{|z_{\max}|, |z_{\min}|\}$, then the difference can be bounded as follows:
\begin{align*}
&|\Pi_\calC(b_{c, \gamma})_{\#}\eta_{N, \infty}(z_i) - \Pi_\calC(b_{c, \gamma})_{\#}\eta_{N, k}(z_i)|\\
\leq & \frac{|c - z_{i-1}|}{|z|}|\sum_{j \in \calL_i}(p_j^{N, \infty}-p_j^{N, k})| + \frac{|z_{i+1}-c|}{|z|}|\sum_{j \in \calR_i}(p_j^{N, \infty}-p_j^{N, k})| \\
&+\frac{\gamma}{|z|}|\sum_{j \in \calL_i}z_j(p_j^{N, \infty}-p_j^{N, k})| + \frac{\gamma}{|z|}|\sum_{j \in \calR_i}z_j(p_j^{N, \infty}-p_j^{N, k})| \\
&\leq \frac{|c-z_{i-1}||\calL_i|}{|z|}\|p^{N, \infty}-p^{N, k}\|_\infty + \frac{|z_{i+1}-c||\calR_i|}{|z|}\|p^{N, \infty}-p^{N, k}\|_\infty \\
&\quad+ \frac{\gamma}{|z|}\sum_{j \in \calL_i}|z_j|\cdot\|p^{N, \infty}-p^{N, k}\|_{\infty} + \frac{\gamma}{|z|}\sum_{j \in \calR_i}|z_j|\cdot\|p^{N, \infty}-p^{N, k}\|_{\infty} \\
&\leq \bigg(\frac{2\beta(\gamma+1)}{\gamma|z|} + \frac{2(\gamma+1)\mu}{|z|}\bigg)\|p^{N, \infty}-p^{N, k}\|_\infty \\
&= \frac{2(\gamma+1)(\beta + \gamma\mu)(N-1)}{\gamma(z_N-z_1)}\|p^{N, \infty}-p^{N, k}\|_\infty \\
&=\delta_\Pi\|p^{N, \infty}-p^{N, k}\|_\infty, \quad\forall i \in [N].
\end{align*}
As a result, we have
\begin{align*}
\|\Pi_\calC(b_{c, \gamma})_{\#}\eta_{N, \infty} - \Pi_\calC(b_{c, \gamma})_{\#}\eta_{N, k}\|_\infty \leq \delta_\Pi\big\|\eta_{N, \infty} - \eta_{N, k}\big\|_\infty\le 2\delta_\Pi C(N,k).
\end{align*}
Repeatedly applying this argument $h$ times yields
\begin{align*}
\bigl\|\Pi_{\mathcal{C}}\bigl(b_{c_1,\gamma}\bigr)_{\#}\ldots\Pi_{\mathcal{C}}\bigl(b_{c_h,\gamma}\bigr)_{\#}\eta_{N,\infty}
-\Pi_{\mathcal{C}}\bigl(b_{c_1,\gamma}\bigr)_{\#}\ldots\Pi_{\mathcal{C}}\bigl(b_{c_h,\gamma}\bigr)_{\#}\eta_{N,k}\bigr\|_{\infty}
\leq (\delta_{\Pi})^h\bigl\|\eta_{N,\infty}-\eta_{N,k}\bigr\|_{\infty}\le 2\delta_\Pi^h C(N,k).
\end{align*}
\end{proof}

\begin{lemma}\label{lemma:ProbMeasureGradError}
Let $k(N, H)$ be the number of times the oracle $\Pi_{\calC}\calT^\pi$ is called, where $k(N, H) = \kappa N(H+1)$, and $H$ is the length of the sampled trajectory. Denote 
$$\nabla_{\theta}\eta^{s_0}_{N, k}:=\mathbb{E}_{\tau_{\theta}}\left[g_{N, k}(s_0) + \sum_{t=1}^{|\tau_{\theta}|}\tilde{\mathcal{B}}^{\tau_{\theta}(s_0,s_t)}g_{N, k}(s_t)\right],$$ where $g_{N,k}(s):=\sum_{a \in \calA}\nabla_\theta\pi_\theta(a|s)\eta_{N,k}^{(s,a)}$. Then, for sufficiently large $N$, we have
\begin{align*}
\big\|\nabla_{\theta}\eta_{N, \infty}^{s_0} - \nabla_{\theta}\eta_{N, k}^{s_0}\big\|_\infty = \calO(N^{0.5}\gamma^{\kappa N/2}).
\end{align*}
\end{lemma}
\begin{proof}
Given a sampled trajectory $\tau_{\theta}=(s_0, a_0, c_0, \dots, s_H)$ of length $H$, we denote $\nabla_{\theta}\eta^{s_0}_{N, \infty}(\tau_{\theta}):=g_{N, \infty}(s_0) + \sum_{t=1}^{|\tau_{\theta}|}\tilde{\mathcal{B}}^{\tau_{\theta}(s_0,s_t)}g_{N, \infty}(s_t)$ and $\nabla_{\theta}\eta^{s_0}_{N, k}(\tau_{\theta}):=g_{N, k}(s_0) + \sum_{t=1}^{|\tau_{\theta}|}\tilde{\mathcal{B}}^{\tau_{\theta}(s_0,s_t)}g_{N, k}(s_t)$ following Theorem \ref{theorem:CatPolicyGrad}, then we have
\begin{align*}
&\|\nabla_{\theta}\eta_{N, \infty}^{s_0}(\tau_{\theta}) - \nabla_{\theta}\eta_{N, k}^{s_0}(\tau_{\theta})\|_\infty\\
\leq &\|g_{N, k}(s_0)-g_{N, \infty}(s_0)\|_\infty + \|\Pi_\calC(b_{c_0, \gamma})_{\#} g_{N, \infty}(s_1)-\Pi_\calC(b_{c_0, \gamma})_{\#}g_{N, k}(s_1)\|_\infty \\
&+\dots + \|\Pi_\calC(b_{c_0, \gamma})_{\#}\dots \Pi_\calC(b_{c_{H-1}, \gamma})_{\#} g_{N, \infty}(s_H) - \Pi_\calC(b_{c_0, \gamma})_{\#}\dots \Pi_\calC(b_{c_{H-1}, \gamma})_{\#} g_{N, k}(s_H)\|_\infty \\
= &2|\calA|\cdot\|\nabla_{\theta}\pi\|_\infty \cdot [C(N, k) + \delta_{\Pi}\cdot C(N, k) + \dots + \delta_{\Pi}^{H}\cdot C(N, k)] \\
= &\frac{2|\calA|\cdot\|\nabla_{\theta}\pi\|_\infty(\delta_{\Pi}^{H+1}-1)}{(\delta_{\Pi}-1)}C(N, k)
\end{align*}
Let the probability of trajectory $\tau_{\theta}$ be $P(\tau_{\theta})$ and the probability of trajectory having length $H$ be $P(|\tau_{\theta}|=H)$, and let 
$k$ be a function of $N$ and $H$ such that $k(N, H) = \kappa N(H+1)$. By Theorem~\ref{theorem:CatPolicyGrad}, the gradient error can be computed as
\begin{align*}
\|\nabla_{\theta}\eta_{N, \infty}^{s_0} - \nabla_{\theta}\eta_{N, k}^{s_0}\|_\infty 
&\le \sum_{\tau_{\theta}} P(\tau_{\theta}) \|\nabla_{\theta}\eta_{N, \infty}^{s_0}(\tau_{\theta}) - \nabla_{\theta}\eta_{N, k}^{s_0}(\tau_{\theta})\|_\infty\\
&\le \sum_{h=1}^{\infty}P(|\tau_{\theta}|=h)\frac{2|\calA|\cdot\|\nabla_{\theta}\pi\|_\infty(\delta_{\Pi}^{h+1}-1)}{(\delta_{\Pi}-1)}C(N, k)\\
&\leq \frac{2|\calA|\cdot\|\nabla_{\theta}\pi\|_{\infty}}{\delta_{\Pi}-1}\sqrt{\frac{\delta_0N}{z_N-z_1}}\sum_{h=1}^{\infty}(\delta_{\Pi}^{h+1}-1)\gamma^{\frac{1}{2}\kappa N (h+1)} \\
&\leq \frac{2\sqrt{N}\cdot|\calA|\cdot\|\nabla_{\theta}\pi\|_\infty}{\delta_\Pi-1}\sqrt{\frac{\delta_0}{z_N-z_1}}\sum_{h=1}^{\infty}(\delta_\Pi\gamma^{\frac{1}{2}\kappa N})^h\\
&\overset{(a)}{=} \frac{2\sqrt{N}\cdot|\calA|\cdot\|\nabla_{\theta}\pi\|_\infty}{\delta_\Pi-1}\sqrt{\frac{\delta_0}{z_N-z_1}}\bigg(\frac{\delta_\Pi \gamma^{\kappa N/2}}{1-\delta_\Pi \gamma^{\kappa N/2}}\bigg)
\end{align*}
where $(a)$ is true because $\delta_\Pi \gamma^{\kappa N/2}<1$ when $N\to\infty$. Since $\delta_{\Pi}=\calO(N)$, for any fixed $\kappa>0$, $\delta_{\Pi}\gamma^{\kappa N/2}\to 0$ as $N\to\infty$. As a result, when $N$ is sufficiently large, we have $\delta_\Pi\gamma^{\kappa N/2}\le q<1$ and $1-\delta_\Pi\gamma^{\kappa N/2} \ge 1-q$. Moreover, when $N$ is sufficiently large, $\delta_{\Pi}/(\delta_{\Pi}-1)$ can be uniformly upper bounded by a constant. Hence, $\|\nabla_{\theta}\eta_{N, \infty}^{s_0} - \nabla_{\theta}\eta_{N, k}^{s_0}\|_\infty = \calO(N^{0.5}\gamma^{\kappa N/2})$.
\end{proof}

\begin{corollary}[$(1-\alpha)$-quantile corollary]
\label{corollary:AlphaQuantileCorollary}
Suppose Assumption~\ref{assumption:AlphaQuantile} holds. Given any policy $\pi$, let $\eta_{N, \infty}$ be the limiting distribution of $\Pi_{\calC}\calT^\pi$ and let $\eta_{N, k}$ be the categorical distribution obtained after $k$ iterations of the operator $\Pi_{\calC}\calT^\pi$, starting from an initial distribution $\eta_{N, 0}$. Let $F^{N, \infty}$ and $F^{N, k}$ denote the CDFs of $\eta_{N, \infty}$ and $\eta_{N, k}$ ($F^{N, \infty}_j = \sum_{i=1}^{j}p_i^{N, \infty}$ and $F^{N, k}_j = \sum_{i=1}^{j}p_i^{N, k}$ for all $j\in [N]$), respectively. Suppose $z_j$ is the $(1-\alpha)$-quantile of $\eta_{N, \infty}$ for some $j \in [N]$. If $\kappa$ in Lemma~\ref{lemma:ProbMeasureGradError} satisfies
\begin{align*}
\kappa > \dfrac{\log\big(\tfrac{N\delta_0}{\epsilon_\alpha^2(z_N-z_1)}\big)}{N\log(1/\gamma)} = \frac{\log(C_1N\epsilon_{\alpha}^{-2})}{N\log(1/\gamma)},
\end{align*}
where $\epsilon_\alpha = \inf_{\theta\in\Theta}\min\{F_j^{N, \infty} - (1-\alpha), \,(1-\alpha) - F_{j-1}^{N, \infty}\}$, then $z_j$ is also the $(1-\alpha)$-quantile of $\eta_{N, k}$.
\end{corollary}
\begin{proof}
Since $F^{N, \infty}_j = \sum_{i=1}^{j}p_i^{N, \infty}$ and $F^{N, k}_j = \sum_{i=1}^{j}p_i^{N, k}$, by Eq.~\eqref{eq:CDF}, we have
\begin{align*}
|F_j^{N, k} - F_j^{N, \infty}| \leq C(N, k) \text{ and } |F_{j-1}^{N, k} - F_{j-1}^{N, \infty}| \leq C(N, k),
\end{align*}
which is equivalent to 
\begin{align*}
&F_j^{N, \infty} - C(N, k) \leq F_j^{N, k} \leq F_j^{N, \infty} + C(N, k), \\
&F_{j-1}^{N, \infty} - C(N, k) \leq F_{j-1}^{N, k} \leq F_{j-1}^{N, \infty} + C(N, k).
\end{align*}
Let $\epsilon_\alpha = \inf_{\theta\in\Theta}\min\{F_j^{N, \infty} - (1-\alpha), \,(1-\alpha) - F_{j-1}^{N, \infty}\}$, then $z_j$ is also the $(1-\alpha)$-quantile of $\eta_{N, k}$, i.e.,
\begin{align*}
F_{j-1}^{N, k} \leq F_{j-1}^{N, \infty} + C(N, k) < (1-\alpha) \text{  and  } (1-\alpha) < F_j^{N, \infty} - C(N, k) \leq F_j^{N, k},
\end{align*}
whenever $C(N, k) < \epsilon_\alpha$, or equivalently, $\kappa > \dfrac{\log\big(\tfrac{\delta_0N}{\epsilon_\alpha^2(z_N-z_1)}\big)}{N\log(1/\gamma)} = \dfrac{\log(C_1N\epsilon_{\alpha}^{-2})}{N\log(1/\gamma)}$. (Note that $k = \kappa N(H+1)$ and $H\ge 0$.)
\end{proof}

\begin{lemma}\label{lemma:CVaR_Grad_Error}
Suppose Assumption~\ref{assumption:AlphaQuantile} holds. Let $\epsilon_\alpha = \inf_{\theta\in\Theta}\min\{\sum_{i=1}^{j}p_i^{N,\infty} - (1-\alpha), \,(1-\alpha) - \sum_{i=1}^{j-1}p_i^{N,\infty}\}$. Denote $\nabla_{\theta}\eta^{s}_{N, k}=\sum_{i=1}^N\tilde{p}_i\delta_{z_i}$ where $\tilde{p}_i:=\nabla_{\theta}\eta^{s}_{N, k}(\{z_i\})$. Let $h$ denote the approximate gradient constructed from Remark~\ref{remark:CatPolicyGradComputation} by replacing all quantities associated with $\eta_{N,\infty}^s$ by the ones associated with $\eta_{N,k}^s$, i.e., $h= \frac{1}{\alpha}
    \sum_{i=1}^{N}\tilde{p}_i\bigl(z_i - \tilde{q}_\alpha^{\theta}\bigr) \mathbf{1}_{\{z_i > \tilde{q}_\alpha^{\theta}\}}$,
where $\tilde{q}_\alpha^{\theta}$ is the $(1-\alpha)$-quantile of $Z_{N,k}^s$. Then the CVaR gradient error is bounded by
\begin{align*}
\|\nabla_{\theta}\rho_{\text{CVaR}}(Z_{N, \infty}) - h\|_2 \leq \epsilon_g
\end{align*}
provided that
\begin{align*}
\kappa > \max\bigg\{\dfrac{2\log(C_2N^{1.5}\epsilon_g^{-1})}{N\log(1/\gamma)}, \dfrac{\log(C_1N\epsilon_\alpha^{-2})}{N\log(1/\gamma)}\bigg\}.
\end{align*}
\end{lemma}
\begin{proof}
By Corollary~\ref{corollary:AlphaQuantileCorollary}, if $\kappa \geq \dfrac{\log(C_1N\epsilon_\alpha^{-2})}{N\log(1/\gamma)}$, we have that both $F^{N, \infty}$ and $F^{N, k}$ have the same $q_\alpha$ (the $(1-\alpha)$-quantile), i.e., $q_{\alpha}=\tilde{q}_{\alpha}$. Let $\calT_{1-\alpha} = \{j: F^{N, \infty}_j > 1-\alpha\}$. Denote $d(\theta)$ as the dimension of $\theta$. From Eq.~\eqref{eq:CVaR_Grad}, we have
\begin{align*}
\|\nabla_{\theta}\rho_{\text{CVaR}}(Z_{N, \infty}) - h\|_2 &\leq\sqrt{d(\theta)}\cdot\bigg\|\frac{1}{\alpha}\sum_{j \in \calT_{1-\alpha}}(z_j-q_\alpha)(\nabla_{\theta}p_j^{N, \infty} - \tilde{p}_j)\bigg\|_\infty \\
&\leq \frac{\sqrt{d(\theta)}\cdot|\calT_{1-\alpha}|}{\alpha}\cdot|z_{\max}-q_\alpha|\cdot \|\nabla_{\theta}\eta_{N, \infty} - \nabla_{\theta}\eta_{N, k}\|_\infty \\
&\le C_2N^{1.5}\gamma^{\kappa N / 2} \leq \epsilon_g
\end{align*}
whenever $\kappa > \dfrac{2\log(C_2N^{1.5}\epsilon_g^{-1})}{N\log(1/\gamma)}$. Overall, we need $\kappa > \max\bigg\{\dfrac{2\log(C_2N^{1.5}\epsilon_g^{-1})}{N\log(1/\gamma)}, \dfrac{\log(C_1N\epsilon_\alpha^{-2})}{N\log(1/\gamma)}\bigg\}$.
\end{proof}

\begin{reptheorem}{\ref{theorem:CDPG_Convergence}}
Suppose Assumption~\ref{assumption:AlphaQuantile} holds. In Algorithm~\ref{algorithm:CDPG}, let the stepsize $\delta=1/\beta$ and the number of $\Pi_{\calC}\calT^\pi$ oracle calls $k(N,|\tau_{\theta}|) = \kappa N|\tau_{\theta}+1|$. For any $\epsilon>0$, we have $\min_{t = 1, \dots, T}\|\nabla_{\theta}\rho_{\text{CVaR}}(Z_{\theta_t, N})\|_2^2 \leq \epsilon$, whenever
\begin{align*}
&T \geq \dfrac{4\beta(\rho_{\text{CVaR}}(Z_{\theta_1, N})-\min_{\theta\in\Theta}\rho_{\text{CVaR}}(Z_{\theta, N}))}{\epsilon}\ \text{ and }\\
&\kappa > \max\bigg\{\dfrac{2\log(C_3N^{1.5}\epsilon^{-0.5})}{N\log(1/\gamma)}, \dfrac{\log(C_1N\epsilon_\alpha^{-2})}{N\log(1/\gamma)}\bigg\},
\end{align*}
where $C_1,\ C_3$ are constants that depend on problem parameters.
\end{reptheorem}
\begin{proof}
By Lemma~\ref{lemma:CVaRSmooth}, CVaR is $\beta$-smooth. Let $\rho_t := \rho_{\text{CVaR}}(Z_{\theta_t, N})$ and $\rho^*=\min_{\theta}\rho_{\text{CVaR}}(Z_{\theta,N})$. By Lemma~\ref{lemma:CVaR_Grad_Error}, the gradient error is bounded by $\epsilon_g$ when $\kappa >\max\bigg\{\dfrac{2\log(C_2N^{1.5}\epsilon_g^{-1})}{N\log(1/\gamma)}, \dfrac{\log(C_1N\epsilon_\alpha^{-2})}{N\log(1/\gamma)}\bigg\}$. Then by Lemma~\ref{lemma:ConvergenceWithError}, we have
\begin{align*}
\frac{1}{T}\sum_{t=1}^{T}\|\nabla_{\theta}\rho_t\|_2^2 \leq \frac{2\beta(\rho_1 - \rho^*)}{T} + \epsilon_g^2.
\end{align*}
Furthermore, let $\epsilon_g^2 = \frac{1}{2}\epsilon$,
then $\kappa$ is required to be
\begin{align*}
\kappa > \max\bigg\{\dfrac{2\log(C_3N^{1.5}\epsilon^{-0.5})}{N\log(1/\gamma)}, \dfrac{\log(C_1N\epsilon_\alpha^{-2})}{N\log(1/\gamma)}\bigg\}
\end{align*}
If we further let $T \geq \frac{4\beta(\rho_1-\rho^*)}{\epsilon}$, we then have
\begin{align*}
\min_{t = 1, \dots, T}\|\nabla_{\theta}\rho_t\|_2^2 \leq \frac{1}{T}\sum_{t=1}^{T}\|\nabla_{\theta}\rho_t\|_2^2 \leq \frac{2\beta(\rho_1 - \rho^*)}{T} + \epsilon_g^2\le\frac{\epsilon}{2}+\frac{\epsilon}{2}\le \epsilon,
\end{align*}
which completes the proof.
\end{proof}

\end{document}

%% file: defs.tex
\def\bbE{{\mathbb{E}}}

\def\bbP{{\mathbb{P}}}

\def\bbR{{\mathbb{R}}}

\def\calA{{\mathcal{A}}}
\def\calB{{\mathcal{B}}}
\def\calC{{\mathcal{C}}}

\def\calE{{\mathcal{E}}}
\def\calF{{\mathcal{F}}}

\def\calJ{{\mathcal{J}}}

\def\calL{{\mathcal{L}}}
\def\calM{{\mathcal{M}}}

\def\calO{{\mathcal{O}}}
\def\calP{{\mathcal{P}}}

\def\calR{{\mathcal{R}}}
\def\calS{{\mathcal{S}}}
\def\calT{{\mathcal{T}}}

\def\calZ{{\mathcal{Z}}}

